\documentclass[11pt]{article} 
\usepackage{times}

\usepackage{microtype}
\usepackage{graphicx}
\usepackage{subfigure}
\usepackage{booktabs} 

\RequirePackage[authoryear,round]{natbib}
\usepackage{fullpage}
\usepackage[utf8]{inputenc}
\usepackage[T1]{fontenc}
\usepackage{microtype}
\usepackage{csquotes}
\usepackage[osf,sc]{mathpazo}
\RequirePackage[scaled=0.90]{helvet}
\RequirePackage[scaled=0.85]{beramono}
\RequirePackage{textcomp}
\usepackage{xcolor}
\definecolor{DarkGreen}{rgb}{0.075,0.375,0.075}
\definecolor{DarkRed}{rgb}{0.5,0.1,0.1}
\definecolor{DarkBlue}{rgb}{0.1,0.1,0.5}
\definecolor{Gray}{rgb}{0.2,0.2,0.2}
\usepackage[pdftex]{hyperref}
\hypersetup{
    unicode=false,          
    pdftoolbar=true,        
    pdfmenubar=true,        
    pdffitwindow=false,      
    pdfnewwindow=true,      
    colorlinks=true,       
    linkcolor=DarkBlue,          
    citecolor=DarkGreen,        
    filecolor=DarkRed,      
    urlcolor=DarkBlue,          
    pdftitle={},
    pdfauthor={},
}
\usepackage[skip=2mm,indent=5mm]{parskip}
\usepackage[small]{caption}
\RequirePackage{amsthm,amsmath,amsfonts,amssymb}

\usepackage[capitalize,noabbrev]{cleveref}

\usepackage{amsmath,amsfonts,bm}

\def\eqref#1{equation~\ref{#1}}

\def\1{\bm{1}}

\DeclareMathAlphabet{\mathsfit}{\encodingdefault}{\sfdefault}{m}{sl}
\SetMathAlphabet{\mathsfit}{bold}{\encodingdefault}{\sfdefault}{bx}{n}

\def\gF{{\mathcal{F}}}
\def\gG{{\mathcal{G}}}

\def\gX{{\mathcal{X}}}
\def\gY{{\mathcal{Y}}}

\newcommand{\E}{\mathbb{E}}

\DeclareMathOperator*{\argmax}{arg\,max}

\usepackage{hyperref}
\usepackage{url}
\usepackage{wrapfig}

\title{Predictable performance is suboptimal\\
and what it means for model routing}

\author{%
  Guanhua Zhang\footnote{Corresponding author: \href{mailto:guanhua.zhang@tuebingen.mpg.de}{guanhua.zhang@tuebingen.mpg.de}}, Moritz Hardt
}
\date{
\textit{Max Planck Institute for Intelligent Systems, Tübingen} \\
\textit{Tübingen AI Center}
}

\usepackage{mathtools}
\usepackage{amsthm}
\usepackage{dsfont}
\usepackage{pifont}

\newcommand{\given}{\,\vert\,}
\newcommand{\falt}{f_{\mathrm{Alt}}}
\newcommand{\Lalt}{L_{Y \mid Z=0}}
\newcommand{\Ly}{L_{Y}}
\newcommand{\Lz}{L_{Z}}
\newcommand{\pred}{\hat{f}}
\newcommand{\recal}[1]{\mathcal{C}_{g}(#1)}
\renewcommand{\1}{\mathds{1}}

\newcommand{\model}[1]{\textsc{#1}}
\newcommand{\dataset}[1]{\texttt{#1}}

\newtheorem{theorem}{Theorem}
\newtheorem{proposition}{Proposition}
\newtheorem{lemma}{Lemma}
\newtheorem{assumption}{Assumption}
\newtheorem{corollary}{Corollary}

\begin{document}

\maketitle

\begin{abstract}
Large language models (LLMs) differ in capabilities and costs, motivating routing methods that adaptively choose a model for each query.
These methods rely on \emph{performance prediction}: given a query, estimate whether a model will answer it correctly.
Most routing methods train separate predictors of candidate models' performance to guide model selection.
Our work shows that performance prediction only beats a simple baseline when the candidate model is suboptimal on the task or the performance predictor adds significant capacity relative to the candidate model.
Specifically, we take a model's self-confidence, the probability it assigns to its chosen answer, as a baseline and study how much a performance predictor can improve on it.
First, we show that any performance predictor that outperforms a model's self-confidence can be used to improve the model's loss: the predictor's advantage transfers exactly in log loss through recalibration.
Second, accounting for the predictor's additional parameters, we bound its advantage by the model's suboptimality plus the gain from those parameters.
Third, we extend our theoretical results to two-model cascades, which route queries to a more capable model when the first model is predicted to fail.
Experiments with 16 language models on four benchmarks show that predictors gain a larger advantage over self-confidence on more suboptimal models.
This advantage shrinks as models improve through finetuning.
Cascading exhibits the same two patterns: predictor-based routing offers larger gains over confidence-based routing for more suboptimal models, and these gains diminish after finetuning.
Code is available at \url{https://github.com/socialfoundations/performance-prediction}.
\end{abstract}

\section{Introduction}
\label{sec:intro}

\begin{figure}[t]
  \centering
  \includegraphics[width=\linewidth]{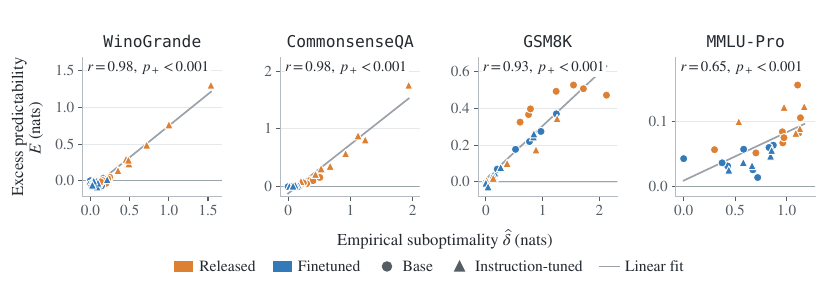}
  \caption{\textbf{Predictors have a larger advantage over self-confidence on more suboptimal models.}
  Each point is a released (orange) or finetuned (blue) model.
  Empirical suboptimality is the gap to the best observed answer log loss on each task.
  Excess predictability $E$ is the reduction in correctness log loss achieved by a separately trained 4B predictor over self-confidence.
  Colors identify training stage; shapes identify the release type.
  Lines are least-squares fits; $r$ is Pearson correlation and $p_+$ its one-sided $p$-value for positive correlation.
  See Section~\ref{sec:exp-setup} for experimental details.
  }
  \label{fig:exp-predictability}
\end{figure}

Performance prediction estimates whether a model answers a query correctly~\citep{kadavath2022language,shnitzer2023routing}.
These predictions can help routing systems decide when to trust an answer and when to call a stronger model~\citep{geifman2017selective,chen2023frugalgpt}.
By catching a weak model's errors accurately, routing can achieve better performance at the same cost.
But what makes a model's errors predictable, and what does that reveal about the model?

In this paper, we study the predictability of a model's performance relative to a simple baseline: the model's \emph{self-confidence}, the probability it assigns to its chosen answer.
Our key observation is that \emph{any predictor that beats self-confidence can improve the model by the same amount}.
This means that a predictor's advantage over self-confidence is limited by the model's room for improvement.
We formalize and test this connection, and extend the analysis to routing.

\subsection{Our Contributions}

\textbf{Suboptimality bounds excess predictability.}
We connect performance prediction and model improvement through recalibration: replacing the model's confidence in its chosen answer with the predictor's estimate, and normalising the rest.
We call the predictor's reduction in correctness log loss relative to self-confidence \emph{excess predictability}.
Recalibration reduces the model's answer log loss by exactly the same amount (Proposition~\ref{prop:transfer}).
We assume a \emph{scaling floor}: the lowest answer loss achievable on a fixed task at each parameter budget (Assumption~\ref{a:floor}).
The recalibrated model cannot beat the floor at the combined parameter budget of the model and predictor.
This gives the bound (Theorem~\ref{thm:predictor-floor})
\[
\boxed{\;
  \text{\itshape excess predictability}
  \;\leq\;
  \text{\itshape suboptimality}
  \; + \;
  \text{\itshape capacity gain}
\;}.
\]
Here, \emph{suboptimality} is the model's gap to the floor at its own budget.
The \emph{capacity gain} is the decrease in the floor from adding the predictor's parameters.

\textbf{Greater excess predictability comes with greater suboptimality.}
Our upper bound does not imply that more suboptimal models are necessarily easier to predict beyond self-confidence.
Empirically, however, we observe a positive association between the two.
We conduct experiments on \dataset{WinoGrande}, \dataset{CommonsenseQA}, \dataset{GSM8K}, and \dataset{MMLU-Pro}, using 16 language models from eight families, ranging from 7 to 12B parameters.
Using the gap to the best observed answer loss on each task as a proxy for suboptimality, we find a positive association with excess predictability on all four tasks, with Pearson correlations of 0.98, 0.98, 0.93, and 0.65, respectively (Figure~\ref{fig:exp-predictability}).
We also finetune the models, retaining the 59 model-task pairs whose development accuracy improves.
Answer loss falls in all 59 retained pairs and excess predictability falls in 55 (Figure~\ref{fig:exp-finetuning-changes}), despite training separate 4B predictors for the released and finetuned models.

\textbf{Suboptimality bounds error capture beyond confidence routing.}
We extend our analysis to two-model cascades, where a router decides when to escalate a query to a stronger model.
We compare a learned router with confidence routing, which escalates the weak model's lowest-confidence queries, at the same escalation rate.
Any router that captures more weak-model errors can be converted into a predictor with positive excess predictability, even if the router does not explicitly estimate correctness (Lemma~\ref{lem:router-to-predictor}).
This bounds the additionally captured weak-model errors by weak-model suboptimality and the capacity gain from the router's added parameters (Theorem~\ref{thm:weak-error-capture}).
The cascade accuracy gain equals the additional weak-model errors captured minus the additional strong-model errors encountered.
The same bound applies to the cascade accuracy gain if the learned router encounters at least as many strong-model errors as confidence routing (Corollary~\ref{cor:cascade-improvement}).

\textbf{More suboptimal models have larger routing advantages.}
Using the same models and predictors, we compare predictor-based routing with confidence routing at equal escalation rates.
We measure the \emph{routing advantage} as the gain in cascade accuracy, averaged over escalation rates.
Models with greater empirical suboptimality tend to have larger routing advantages on all four tasks, with Pearson correlations between 0.66 and 0.91 (Figure~\ref{fig:exp-cascade-advantage}).
This advantage shrinks after finetuning in 53 of the 59 retained pairs (Section~\ref{sec:exp-cascading}).

\textbf{Main takeaway.} Together, these results suggest three sources of predictability: the model's self-confidence, its suboptimality, and the predictor's added capacity. 
Improving the underlying model generally reduces the additional value of learned prediction and routing beyond self-confidence.

\section{Related Work}
\label{sec:related}

\textbf{Performance prediction.}
Performance predictors, also called assessors, estimate whether a system will succeed on a given input \citep{hernandezorallo2022assessors,zhou2022reject}.
They use model outputs \citep{yao2024correctness}, reference evaluations and question features \citep{pacchiardi2024instances}, follow-up queries \citep{sam2025followup}, or internal activations before the model answers \citep{cencerrado2026noanswer}.
\citet{kapoor2024taught} finetune language models on graded answers to predict their own correctness and that of other models.
\citet{xiao2026correctness} find that models predict their own correctness no better than unrelated models do, and train predictors that work across models.
Beyond binary correctness, \citet{hsu2026pie} estimate continuous quality scores and uncertainty intervals for free-form generation.
At the benchmark level, studies forecast performance from training compute \citep{owen2024predictable} or scores on other benchmarks \citep{ruan2024observational}.
Others estimate benchmark scores from a small set of evaluated examples \citep{maiapolo2024tinybenchmarks,zhang2025benchmark}.
We study where predictability on individual inputs comes from, using self-confidence as the baseline and correctness log loss to measure a predictor's advantage.

\textbf{Model routing and cascading.}
Performance predictions help routing systems balance cost and performance by choosing which models to use for each query \citep{ong2024routellm}.
We use routing broadly to include cascading, where a cheaper model answers first and a router uses its output to decide whether to call a stronger model \citep{chen2023frugalgpt,dekoninck2025unified}.
Routers learn which model to use from preference data, as in RouteLLM \citep{ong2024routellm}, benchmark results \citep{shnitzer2023routing,mohammadshahi2024routoo}, or reward-model scores \citep{lu2024routing}.
Some predict query difficulty to guide model selection \citep{ding2024hybrid}.
RouterBench provides a common framework for evaluating routing methods \citep{hu2024routerbench}.
In cascades, routing rules often escalate to a larger model when the first model has low confidence or gives inconsistent answers \citep{varshney2022cascading,yue2024cascades,chuang2025confident}.
Learned routing rules use token-level uncertainty and model embeddings \citep{gupta2024cascades} or verify the first model's answer before deciding whether to escalate \citep{aggarwal2024automix}.
\citet{wang2024cascadeaware} also improve cascades by training the first model to account for deferral.
\citet{dekoninck2025unified} derive an optimal cascading policy and prove that an existing routing policy is optimal; both depend on good estimates of answer quality.
\citet{jitkrittum2023confidence} show when confidence alone gives an optimal deferral rule and when a learned rule \citep{narasimhan2022posthoc} is needed.
We bound how many more of the first model's errors a learned router can catch than confidence routing under the same escalation rate.

\textbf{Confidence and selective prediction.}
Language models can estimate whether their answers are correct \citep{kadavath2022language} through stated confidence \citep{tian2023just,xiong2024can} or agreement among sampled answers \citep{wang2023selfconsistency,kuhn2023semantic}.
Learned verifiers score candidate solutions to select better answers \citep{cobbe2021verifiers,lightman2023verify}, and the gap between generation and verification abilities is linked to self-improvement \citep{song2025mind}.
Confidence also guides decisions about when to answer.
Selective prediction uses it to decide when to abstain \citep{chow1970optimum,geifman2017selective}, while learning to defer also considers whether another expert will be correct \citep{mozannar2020defer}.
For selective prediction, both calibration and how confidence ranks inputs matter \citep{rabanser2025selective}.
Inputs with the same confidence can still differ in their probability of correctness; grouping loss measures this variation \citep{kull2015novel,perezlebel2023beyond}.
\citet{melo2026epistemic} estimate grouping loss and excess decision risk to guide LLM cascades, and find that grouping loss decreases with model scale and instruction tuning.
Our experiments examine how gains over confidence vary with answer-loss suboptimality across models of similar size and before and after task-specific finetuning.

\textbf{Calibration and loss minimization.}
Language models can produce informative but poorly calibrated risk scores \citep{cruz2024evaluating}.
Post-hoc calibration adjusts confidence to match observed correctness \citep{platt2000probabilistic,guo2017calibration,kumar2019verified}.
Multicalibration extends this requirement to identifiable groups \citep{hebertjohnson2018multicalibration}, including those defined by LLM embeddings and self-annotations \citep{detommaso2024multicalibration}.
Loss minimization and multicalibration theory show how correcting predictable errors in probability estimates can reduce loss, including squared loss and cross-entropy \citep{gopalan2022omnipredictors,gopalan2023loss,blasiok2024loss}.
We show that any reduction in correctness log loss over self-confidence yields an equal reduction in answer log loss.
Under our assumptions for combining models and counting parameters, this gain is bounded by the original model's suboptimality plus the improvement possible with the added parameters.

\section{Predictable Performance is Suboptimal}
\label{sec:main}
We show that gains over self-confidence transfer exactly to lower answer loss and are bounded by model suboptimality and added predictor capacity.
Proofs are deferred to Appendix~\ref{app:proofs-main}.

\subsection{Problem Definition}
Let $(X,Y)$ be drawn from a fixed population, where $X \in \gX$ is the query and $Y \in \gY = \{0,\dots,K-1\}$ is the answer.
A \emph{model} $f \in \gF$ is defined as a conditional probability mass function $f(y|x)$.
Its point prediction is $\pred(x)=\argmax_y f(y\given x)$, with ties broken by a fixed rule.
We score any model $f$ by its \emph{answer loss} $\Ly(f) = \E[-\log f(Y \given X)]$.

We define the correctness indicator $Z=\1\{Y = \pred(X)\} \in \{0,1\}$.
A \emph{performance predictor} $g \in \gG$ is defined as a conditional probability mass function $g(z\given x)$ that may read $f$'s output on $x$.
We require full support, $g(z\given x)>0$, and score $g$ by its \emph{correctness loss} $\Lz(g,f) = \E[-\log g(Z \given X) ]$.

The model's self-confidence provides a natural baseline: any model $f$ induces a performance predictor $f_Z$ with $f_Z(1 \given x) = f(\pred(x) \given x)$.
We abbreviate $\Lz(f) \coloneqq \Lz(f_Z,f)$.

The goal of performance prediction is to minimize $\Lz(g,f)$, where $\Lz(f)$ serves as a baseline.
We call $\Lz(f)-\Lz(g,f)$ the \emph{excess predictability} achieved by $g$ relative to $f$'s self-confidence; a performance predictor is \emph{effective} when its excess predictability is positive.

\subsection{Any Effective Performance Predictor Improves the Model}
We use $g$ to change how much probability $f$ assigns to its chosen answer, while preserving the relative probabilities of the other answers.

Assuming $f_Z(0\given x)>0$, define the alternative distribution and the \emph{recalibrated model} by
\begin{align}
    \falt(y\given x)&=f(y\given x)/f_Z(0\given x),
    \qquad y\neq\pred(x), \label{eq:def_falt}\\
  \recal{f}(y\given x)&=\begin{cases}
      g(1\given x) & y=\pred(x),\\
      g(0\given x)\falt(y\given x) & y\neq\pred(x)
  \end{cases}.\label{eq:def_recal}
\end{align}
Thus, $\recal{f}$ assigns mass $g(1\given x)$ to $\pred(x)$ and distributes the remaining mass among the other answers in their original proportions.
Since only the correctness probabilities change, the following proposition transfers the predictor's improvement exactly to the answer loss.

\begin{proposition}[Exact transfer]\label{prop:transfer}
For any $f$ with $\Ly(f) < \infty$ and any performance predictor $g$,
\begin{equation}\label{eq:transfer}
  \boxed{\Ly \big(\recal{f}\big) = \Ly(f) - \big(\Lz(f) - \Lz(g,f)\big)}.
\end{equation}
\end{proposition}

Recalibrating with an effective $g$ strictly lowers the answer loss by exactly its excess predictability over $f$'s self-confidence, $\Lz(f)-\Lz(g,f)>0$.
Recalibration is also hard to beat: any model with lower answer loss than $\recal{f}$ must predict $f$'s correctness better than $g$, or predict better among the alternative answers when $f$ is wrong (Proposition~\ref{prop:better-component} in Appendix~\ref{app:better-component}).

\subsection{Predictable Performance is Suboptimal}
\label{sec:parameter-budgets}

For parameter budgets $P,Q\ge 0$, let $\gF_P\subseteq\gF$ denote the models with at most $P$ parameters, and let $\gG_Q\subseteq\gG$ denote the performance predictors with at most $Q$ parameters.
Both sets are therefore nested in their respective budgets.

We allow models in $\gF$ to use performance predictors and routers as components, and apply the same convention to predictors in $\gG$.
We count all parameters, including frozen weights, but count shared parameters only once.
The fixed arithmetic and decision operations used below add no parameters.
Thus, recalibrating $f\in\gF_P$ with $g\in\gG_Q$ gives a model $\recal{f}\in\gF_{P+Q}$.
The loss floor and suboptimality statements below apply to the full set $\gF$ on the fixed task distribution of $(X,Y)$.

Empirical neural scaling laws suggest that the best loss achievable follows a predictable frontier as model size increases~\citep{kaplan2020scaling,hoffmann2022compute}.
We idealize this as follows.

\begin{assumption}[Scaling floor]\label{a:floor}
For each budget $P$, some model $f_P^{*}\in\gF_P$ achieves the lowest loss $\Ly^{*}(P)$ in $\gF_P$:
\begin{equation}\label{eq:scaling-floor}
  \Ly(f)\ge\Ly^{*}(P)
  =\Ly(f_P^{*})
  \qquad\text{for every }f\in\gF_P.
\end{equation}
\end{assumption}

The floor $\Ly^{*}(P)$ is non-increasing in $P$, since a larger budget includes all models available at a smaller budget.
Relative to this floor, define the \emph{suboptimality at budget $P$} of $f\in\gF_P$ as
\begin{equation}\label{eq:excess-loss}
  \delta_P(f)\coloneqq\Ly(f)-\Ly^{*}(P)\ge 0.
\end{equation}
The experiments use the gap to the best observed loss in a cohort of similar-size models as a proxy for this suboptimality (Section~\ref{sec:exp-suboptimality}).

The recalibrated model uses at most $P+Q$ parameters, so its loss cannot fall below $\Ly^{*}(P+Q)$.
Together with exact transfer, this limits how much $g$ can improve on $f$'s self-confidence.

\begin{theorem}[Excess predictability bounded by suboptimality and added capacity]
\label{thm:predictor-floor}
Under Assumption~\ref{a:floor}, every $f\in\gF_P$ with
$\Ly(f)<\infty$ and every $g\in\gG_Q$ satisfy
\begin{equation}\label{eq:edge-cap}
  \boxed{
  \Lz(f)-\Lz(g,f)
  \le
  \big[\Ly^{*}(P)-\Ly^{*}(P+Q)\big]
  +\delta_P(f)
  }.
\end{equation}
\end{theorem}

Equation~\ref{eq:edge-cap} separates two sources of excess predictability.
For the loss-optimal model $f_P^{*}$, a predictor can recover at most the \emph{scaling gain} $\Ly^{*}(P)-\Ly^{*}(P+Q)$, the decrease in the loss floor from adding $Q$ parameters.
A model farther from the floor increases this recoverable excess predictability by at most its suboptimality $\delta_P(f)$.
Any improvement over $f$'s self-confidence that exceeds the scaling gain therefore certifies $\delta_P(f)>0$.
In particular, if the performance predictor adds no parameters, then strictly positive excess predictability implies $\delta_P(f)>0$.

\paragraph{Additional results.}
First, under diminishing returns, a predictor much smaller than the model has excess predictability bounded by the model's suboptimality plus a small capacity correction (Corollary~\ref{cor:diminishing-returns}, Appendix~\ref{app:diminishing-returns}).
Second, suboptimality does not determine excess predictability: lowering answer loss does not necessarily lower correctness loss.
Models can have the same maximum excess predictability yet differ in suboptimality because they assign different probabilities to the answers they do not choose (Proposition~\ref{prop:alternative-invariance}, Appendix~\ref{sec:tightness}).

\subsection{Experiments}
\label{sec:experiments}

Motivated by the preceding bound on excess predictability, we study two empirical questions on 16 language models across four benchmarks.
Do more suboptimal models exhibit greater excess predictability over self-confidence?
Does task-specific finetuning reduce this advantage?

\subsubsection{Setup}
\label{sec:exp-setup}

\textbf{Models and tasks.}
We evaluate 16 base and instruction-tuned models from eight families, with 7 to 12B parameters.
We select four tasks, 
\dataset{WinoGrande} \citep{sakaguchi2019winogrande}, \dataset{CommonsenseQA} \citep{talmor2019commonsenseqa}, \dataset{GSM8K} \citep{cobbe2021verifiers} and \dataset{MMLU-Pro} \citep{wang2024mmlupro}.
\dataset{WinoGrande} and \dataset{CommonsenseQA} are evaluated by scoring answer options, and \dataset{GSM8K} and \dataset{MMLU-Pro} are evaluated by generating answers.
The binary task \dataset{WinoGrande} is a sanity check: predicting a fixed model's correctness is equivalent to predicting the answer.
Appendix~\ref{app:exp-models} lists model releases and data splits, shared across models.

\textbf{Evaluation.}
We use five-shot evaluation with \texttt{lm-eval-harness} \citep{gao2024harness}.
The answer distribution $f(\cdot\given x)$ comes from normalized option likelihoods on the first two tasks and answer frequencies from 128 sampled solutions per question on the generation tasks.
We select the highest-probability answer $\pred$ and use its probability as self-confidence.
On test questions, we measure answer log loss $L_Y(f)$ and correctness log loss $L_Z(f)$; cascades are evaluated by accuracy.
Appendix~\ref{app:exp-losses} details scoring, sampling, answer extraction, and probability estimation.

\textbf{Finetuning.}
We use LoRA \citep{hu2021lora} to finetune on gold answers for the option-scoring tasks and on reasoning traces by \model{Qwen3.5-122B-A10B} for the generation tasks, selecting checkpoints on development performance.
Our experiment design is built upon the idea that task-specific finetuning reduces suboptimality.
Therefore, for each task, we consider all models whose development accuracy does not improve after finetuning to be anomalous and remove the corresponding model pairs from the analysis.
This leaves 16, 16, 14, and 13 model pairs for the respective tasks.
Appendix~\ref{app:exp-training} gives training and retention details.

\textbf{Suboptimality.}
For each task, let $\mathcal C$ contain the released and finetuned versions of the retained models.
We use their best observed answer loss as a proxy for the loss floor at similar model sizes:
\begin{equation}
  \widehat\delta(f)=L_Y(f)-\min_{h\in\mathcal C}L_Y(h).\label{eq:exp-endpoints}
\end{equation}
This reference is fixed for all models, so changes in empirical suboptimality equal changes in answer loss; changing the reference shifts all gaps equally and leaves correlations unchanged.

\textbf{Performance predictor.}
We train a separate \model{Qwen3-Embedding-4B} predictor \citep{zhang2025qwen3embedding} for each model, task, and finetuning stage.
Its input contains the question, the model's final answer, option probabilities or sampled-answer frequencies, and, for generation tasks, a sampled reasoning trace ending in the majority answer.
A linear head on the embedding corrects the self-confidence logit; we train the head and LoRA adapters from a Platt-calibrated initialization~\citep{platt2000probabilistic}.
We fit predictors on 85\% of the development set and select the checkpoint and seed with lowest correctness loss on the remaining 15\%.
On the test set, excess predictability is
\begin{equation}
  E(f)=L_Z(f)-L_Z(g,f).\label{eq:exp-edge}
\end{equation}
Appendix~\ref{app:exp-predictor} gives the architecture and training details.

\subsubsection{Results}
\label{sec:exp-suboptimality}

\begin{figure}[!t]
  \centering
  \includegraphics[width=\linewidth]{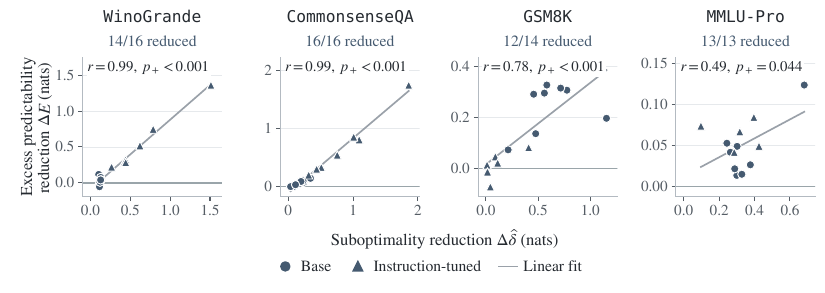}
  \caption{\textbf{Finetuning reduces excess predictability in 55 of 59 model-task pairs.}
  Each point compares one released model with its selected finetuned version; both axes show released minus finetuned values.
  $\Delta\widehat\delta$ is the reduction in answer log loss, and $\Delta E$ is the reduction in excess predictability.
  Points above zero have lower excess predictability after finetuning; counts report these pairs.
  Suboptimality decreases in all 59 pairs.
  Shapes identify the release type; lines are least-squares fits; $r$ is Pearson correlation and $p_+$ its one-sided $p$-value for positive correlation.}
  \label{fig:exp-finetuning-changes}
\end{figure}

\textbf{Greater excess predictability comes with greater suboptimality.}
Figure~\ref{fig:exp-predictability} shows a positive correlation between empirical suboptimality and excess predictability across all four tasks.
We plot released and finetuned models together because the relationship concerns a model's suboptimality, regardless of how it was trained.
Appendix~\ref{app:exp-groups} shows the two groups separately; the association also holds among released models alone.
The association is stronger on the simpler tasks, \dataset{WinoGrande} and \dataset{CommonsenseQA}, with Pearson correlations of 0.98 for both; on \dataset{WinoGrande} this is the expected sanity-check outcome (Section~\ref{sec:exp-setup}).
On the generation-based tasks, \dataset{GSM8K} and \dataset{MMLU-Pro}, the correlations are 0.93 and 0.65, respectively.
The weaker association on \dataset{MMLU-Pro} may partly reflect variation in alternative loss: models can differ in answer loss without differing in excess predictability (Appendix~\ref{sec:tightness}).

\begingroup
\setlength{\intextsep}{6pt}
\begin{wrapfigure}{R}{0.4\textwidth}
\vspace{-3mm}
  \centering
  \setlength{\abovecaptionskip}{4pt}
  \includegraphics[width=0.95\linewidth]{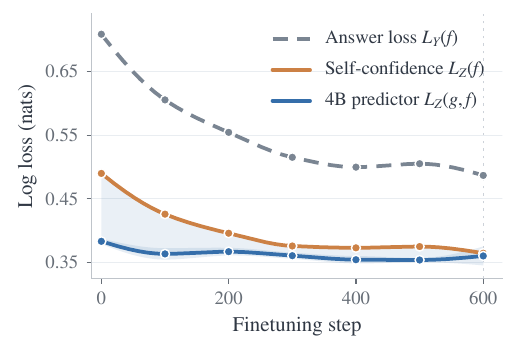}
  \caption{Answer and correctness losses during finetuning of \model{Llama-3.1-8B} on \dataset{CommonsenseQA}.
  Shading shows excess predictability $E(f)$; predictor loss averages three seeds.}
  \label{fig:exp-trajectory}
\vspace{-5mm}
\end{wrapfigure}

\textbf{Finetuning reduces excess predictability.}
Excess predictability falls in 55 of the 59 retained model-task pairs after finetuning (Figure~\ref{fig:exp-finetuning-changes}): 14 of 16 on \dataset{WinoGrande}, all 16 on \dataset{CommonsenseQA}, 12 of 14 on \dataset{GSM8K}, and all 13 on \dataset{MMLU-Pro}.
The four exceptions have relatively small reductions in suboptimality, indicating that finetuning improves their answer loss only modestly.
Note that we train a new predictor for each finetuned model, yet its gain over self-confidence still shrinks.
Larger reductions in suboptimality tend to accompany larger reductions in excess predictability, with Pearson correlations of 0.99, 0.99, 0.78, and 0.49 in task order.

During finetuning of \model{Llama-3.1-8B} on \dataset{CommonsenseQA}, self-confidence catches up with separately retrained predictors: excess predictability falls from 0.107 to 0.004 nats (Figure~\ref{fig:exp-trajectory}).
Appendix~\ref{app:exp-training} gives the trajectory setup.

\textbf{Predictor capacity.}
We compare how much excess predictability 0.6B, 4B, and 8B predictors capture on the same models (Appendix~\ref{app:experiment-capacity}).
Larger predictors generally capture more: the 0.6B predictor adds little beyond Platt-calibrated self-confidence, while the 4B and 8B predictors generally improve on it.
Our main finding holds at all three predictor sizes: finetuning reduces mean excess predictability on every task.
For the 4B predictor, mean gains over Platt-calibrated self-confidence also decrease after finetuning on every task (Appendix Figure~\ref{fig:exp-size-ablation}).
\par
\endgroup

\section{Suboptimality Bounds Cascading}
\label{sec:cascading}
In a cascade, a router decides which queries to send from a weak model to a stronger model.
We compare a learned router with the weak model's self-confidence at the same escalation rate.
Cascades are judged by accuracy at that rate, so this section bounds error rates rather than log loss.
The key connection is that a router that finds more weak-model errors also supplies a better performance predictor.
Proofs are deferred to Appendices~\ref{app:routing-to-prediction}--\ref{app:cascade-advantage}.

\subsection{Cascading Setup}

Let $f_1\in\gF_{P_1}$ and $f_2\in\gF_{P_2}$ be the weak and strong models, with $P_1<P_2$ and $\Ly(f_1)<\infty$.
For $i=1,2$, write $Z_i=\1\{Y=\pred_i(X)\}$ and $\epsilon_i(x)=\Pr(Z_i=0\given X=x)$ for correctness and the actual error probability.
The weak model estimates its error probability by $f_{1,Z}(0\given x)=1-f_1(\pred_1(x)\given x)$, assumed strictly between zero and one almost surely.

A budget-$Q$ router assigns a score $r:\gX\to\mathbb R$ using at most $Q$ additional parameters beyond $f_1$.
At escalation rate $\rho\in[0,1]$, we escalate the highest-scoring $\rho$ fraction of queries to $f_2$, breaking ties uniformly at random.
Its \emph{routing decision} $B_{r,\rho}(x)$ is one for escalation and zero for keeping $f_1$'s answer.
Self-confidence routing applies the same rule to $f_{1,Z}(0\given x)$, giving $B_{f,\rho}$.
At a fixed rate, abbreviate these decisions as $B_r,B_f$, with $\E[B_r]=\E[B_f]=\rho$, and set $u=B_r-B_f$.

Writing $\operatorname{Err}(f)$ for a model's error rate, the difference between the two cascades decomposes as
\begin{equation}\label{eq:relative-cascade-decomposition}
  \begin{aligned}
    \operatorname{Err}(f_{\mathrm{cas},B_r})-\operatorname{Err}(f_{\mathrm{cas},B_f})
    &=-\underbrace{\E[u\epsilon_1(X)]}_{\text{additional weak errors captured}}
    +\underbrace{\E[u\epsilon_2(X)]}_{\text{additional strong errors encountered}}.
  \end{aligned}
\end{equation}
A router can improve on confidence by capturing more weak-model errors, encountering fewer strong-model errors, or both.
We first bound the weak-error contribution, then state when this also bounds the cascade advantage.

\subsection{From Performance Prediction to Cascade Bounds}

At a fixed escalation rate, $B_f$ escalates the lowest-confidence $\rho$ fraction of queries and therefore maximizes the predicted weak errors captured, $\E[B_f f_{1,Z}(0\given X)]$. Thus $\E[u\,f_{1,Z}(0\given X)]\le0$.
A router that captures more actual errors must therefore identify where self-confidence misestimates their probability:
\begin{equation}\label{eq:routing-confidence-residual}
  \begin{aligned}
    \E[u\epsilon_1(X)]
    &=\E[u\,f_{1,Z}(0\given X)]
      +\E[u(\epsilon_1(X)-f_{1,Z}(0\given X))]\\
    &\le\E[u(\epsilon_1(X)-f_{1,Z}(0\given X))].
  \end{aligned}
\end{equation}

We can construct a performance predictor $g_\lambda$ from the routing score $r$ by shifting the weak model's confidence where its routing decisions differ from self-confidence routing:
\begin{equation}\label{eq:routing-confidence-correction}
  \operatorname{logit}g_\lambda(0\given x)
  =\operatorname{logit}f_{1,Z}(0\given x)+\lambda u(x),
\end{equation}
where $\operatorname{logit}(a)=\log(a/(1-a))$ and $g_\lambda(1\given x)=1-g_\lambda(0\given x)$.
Lemma~\ref{lem:router-to-predictor} in Appendix~\ref{app:routing-to-prediction} shows that, if $\E[u\epsilon_1(X)]>0$, then there exists $\lambda>0$ such that $\Lz(g_\lambda,f_1)<\Lz(f_1)$.
In other words, a router that captures additional weak-model errors can be used to construct an effective predictor.

Recalibration with $g_\lambda$ (see Equation~\ref{eq:def_recal}) reuses $f_1$ and adds the budget-$Q$ routing score $r$, the shared escalation rate $\rho$, and the correction strength $\lambda$.
Therefore, the resulting model has at most $P_1+Q+2$ parameters.
Applying Theorem~\ref{thm:predictor-floor} to $f_1$ and $g_\lambda$ bounds the excess predictability.
Combining this upper bound with Lemma~\ref{lem:router-to-predictor} bounds the additional weak errors captured.

\begin{theorem}[Suboptimality and capacity bound weak-error capture beyond confidence]\label{thm:weak-error-capture}
Under Assumption~\ref{a:floor}, every budget-$Q$ routing score $r$ and escalation rate $\rho$ induce decisions $B_r$ and $B_f$ whose difference $u=B_r-B_f$ satisfies
\begin{equation}\label{eq:weak-error-capture-bound}
  \boxed{\begin{aligned}
    &\E[u\epsilon_1(X)]
    \le\sqrt{\min\{\rho,1-\rho\}
      \left[\delta_{P_1}(f_1)+\Ly^*(P_1)-\Ly^*(P_1+Q+2)\right]}
  \end{aligned}}.
\end{equation}
\end{theorem}

At fixed budgets and loss floor, reducing weak-model suboptimality tightens the bound on additional errors captured.
By Equation~\ref{eq:relative-cascade-decomposition}, the same bound applies to cascade accuracy gains when $\E[u\epsilon_2(X)]\ge0$ (Appendix~\ref{app:cascade-advantage}).
Empirically, $\E[u\epsilon_2(X)]$ is nonnegative on average over escalation rates for most evaluated models (Appendix~\ref{app:exp-cascade-results}).

\begin{figure}[!t]
  \centering
  \includegraphics[width=\linewidth]{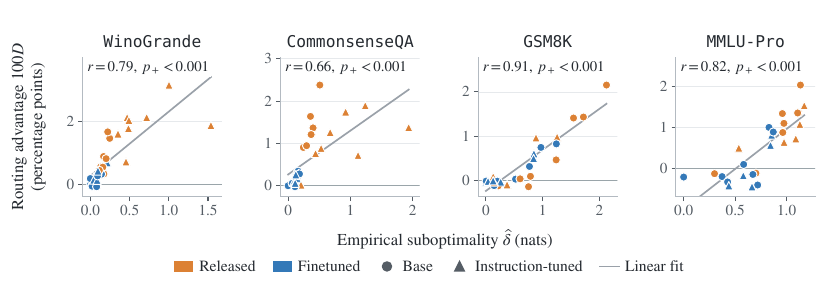}
  \caption{\textbf{More suboptimal models have larger routing advantages.}
  Each point is one released or finetuned weak-model version.
  Empirical suboptimality is the gap to the best observed answer log loss on each task.
  Routing advantage $100D$ is the cascade accuracy gain from a 4B predictor over self-confidence routing to \model{Qwen3.5-122B-A10B}, averaged over equal escalation rates and expressed in percentage points.
  Colors identify training stage; shapes identify the release type.
  Lines are least-squares fits; $r$ is Pearson correlation and $p_+$ its one-sided $p$-value for positive correlation.}
  \label{fig:exp-cascade-advantage}
\end{figure}


\begin{figure}[!t]
  \centering
  \includegraphics[width=\linewidth]{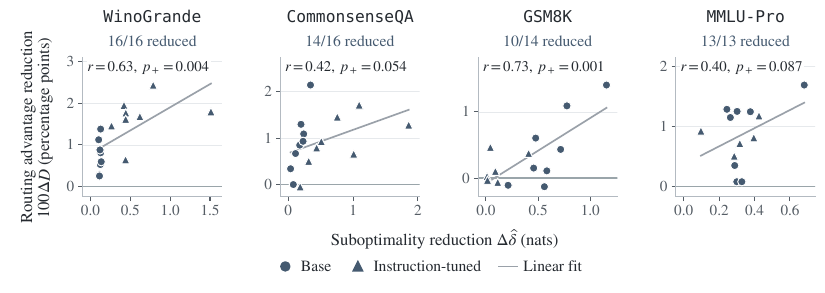}
  \caption{\textbf{Finetuning reduces routing advantage in 53 of 59 model-task pairs.}
  Each point compares one weak model before and after finetuning; both axes show released minus finetuned values.
  $\Delta\widehat\delta$ is the reduction in answer log loss.
  $100\Delta D$ is the reduction in cascading advantage.
  Points above zero have lower routing advantage after finetuning; counts report these pairs.
  Shapes identify the release type; lines are least-squares fits; $r$ is Pearson correlation and $p_+$ its one-sided $p$-value for positive correlation.}
  \label{fig:exp-cascade-reduction}
\end{figure}

\subsection{Experiments}
\label{sec:exp-cascading}

The cascade experiments test a distinct question from predictability: whether predictors rank questions better for escalation and thereby improve cascade accuracy.
For example, Platt calibration with positive slope can reduce correctness log loss relative to raw self-confidence while preserving rankings.
At fixed escalation rates, it therefore leaves routing decisions unchanged and provides no cascade advantage.

\subsubsection{Setup}
Following Section~\ref{sec:cascading}, the weak model $f_1$ answers each question first, and the routing decision determines whether to keep its answer or escalate to the strong model $f_2$.

Each model serves as $f_1$, both before and after finetuning.
We use \model{Qwen3.5-122B-A10B} as $f_2$.
On \dataset{WinoGrande} and \dataset{CommonsenseQA}, we finetune $f_2$ and select the checkpoint with the lowest development answer loss.
On \dataset{GSM8K} and \dataset{MMLU-Pro}, we use its released version, which is already more accurate than all weak models.

We compare the selected 4B predictor $g$ from Section~\ref{sec:exp-setup} with the weak model's self-confidence as sources of routing scores, using error scores $g(0\given x)$ and $f_{1,Z}(0\given x)$.
The thresholding rule from Section~\ref{sec:cascading} gives decisions $B_{g,\rho}$ and $B_{f,\rho}$ at the same escalation rate $\rho$.

The \emph{routing advantage} $D$ measures how much predictor routing improves cascade accuracy over self-confidence routing, averaged over all escalation rates:
\begin{equation}\label{eq:exp-cascade-advantage}
  D=\int_0^1\bigl[\operatorname{Err}(f_{\mathrm{cas},B_{f,\rho}})-\operatorname{Err}(f_{\mathrm{cas},B_{g,\rho}})\bigr]\,d\rho.
\end{equation}
We report $100D$ as the average accuracy gain in percentage points.

\subsubsection{Results}

\begin{wrapfigure}{r}{0.4\textwidth}
  \vspace{-5mm}
  \centering
  \setlength{\abovecaptionskip}{4pt}
  \includegraphics[width=0.95\linewidth]{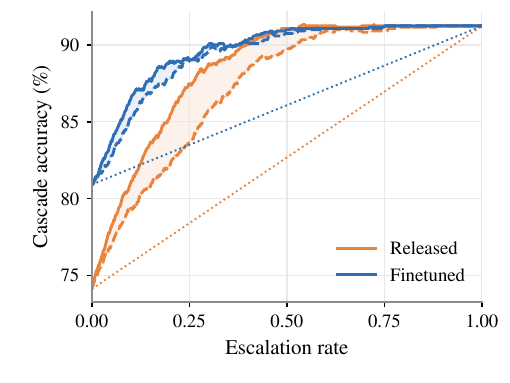}
  \caption{\model{Llama-3.1-8B} cascades on \dataset{CommonsenseQA} before and after finetuning: predictor (solid), confidence (dashed), and random (dotted) routing.
  Shading shows routing advantage $D$.}
  \label{fig:exp-cascade-curves}
  \vspace{-5mm}
\end{wrapfigure}
\textbf{Cascade accuracy across escalation rates.}
For \model{Llama-3.1-8B} on \dataset{CommonsenseQA}, finetuning improves weak-model accuracy and brings self-confidence routing closer to predictor routing (Figure~\ref{fig:exp-cascade-curves}).
Appendix~\ref{app:exp-cascade-curves} shows the curves for all four tasks.

\paragraph{More suboptimal models have larger routing advantages.}
Figure~\ref{fig:exp-cascade-advantage} shows that more suboptimal models have larger routing advantages.
Pearson correlations between empirical suboptimality and routing advantage range from 0.66 to 0.91 across the four tasks.
Appendix~\ref{app:exp-groups} shows released and finetuned models separately.
Compared with excess predictability, routing advantage has a weaker correlation with suboptimality on three of the four tasks.
One possible reason is that finding more weak-model errors does not always improve cascade accuracy: the strong model may also answer these questions incorrectly.
Appendix~\ref{app:exp-cascade-results} separates the routing advantage into additional weak-model errors captured and additional strong-model errors encountered (see Equation~\ref{eq:relative-cascade-decomposition}).
The captured errors alone are at least as strongly associated with suboptimality, with Pearson correlations of 0.80, 0.72, 0.91, and 0.86 in task order.

\paragraph{Finetuning reduces routing advantages.}
Similar to excess predictability, after finetuning, the routing advantage falls in 53 of the 59 retained model-task pairs (Figure~\ref{fig:exp-cascade-reduction}).
Larger reductions in suboptimality tend to accompany larger reductions in routing advantage, with Pearson correlations of 0.40 to 0.73.
Averaged over the retained models, the routing advantage falls from 1.38 to 0.18 percentage points on \dataset{WinoGrande}, 1.09 to 0.18 on \dataset{CommonsenseQA}, 0.52 to 0.20 on \dataset{GSM8K}, and 0.94 to 0.07 on \dataset{MMLU-Pro} (Table~\ref{tab:exp-cascade-split}).

\section{Conclusion}
\label{sec:conclusion}

The value of a separate performance predictor depends on the model it evaluates.
We show that any excess predictability over self-confidence can be turned into an equal reduction in the model's answer loss.
The predictor's advantage is therefore limited by the model's suboptimality plus the capacity the predictor adds, and the same limit governs how many more weak-model errors a router can catch than confidence routing.
In our experiments, more suboptimal models have larger prediction and routing advantages, and finetuning lets self-confidence catch up on both.
Improving the model can thus reduce the need for a separate predictor.

When self-confidence is unavailable, our analysis suggests three sources of predictive value: recovering self-confidence, exploiting suboptimality, and adding capacity.
Recovering self-confidence alone can make a separate predictor useful, even when little room remains to surpass it.

\textbf{Limitations and further directions.}
The theory gives an upper bound, so it does not imply that every reduction in answer loss will reduce a predictor's advantage.
Our cascade bound controls additional weak-model errors captured; bounding the full routing advantage also requires a condition on strong-model errors.
Future work should characterize the strong model's contribution and how to predict the benefit of escalation without running it.
\section*{Acknowledgement}
We would like to thank 
Florian E. Dorner and Nastaran Okati
for helpful discussions and/or feedback on draft versions of this work.


\newpage
\bibliography{references}
\bibliographystyle{plainnat}

\newpage
\appendix
\section{Proofs and Additional Results}
\label{app:proofs}
\subsection{Proofs for Section~\ref{sec:main}}
\label{app:proofs-main}

\begin{proof}[Proof of Proposition~\ref{prop:transfer}]
On a correct prediction, the model assigns the gold answer probability $f_Z(1\given X)$; on an incorrect prediction, it assigns $f_Z(0\given X)\falt(Y\given X)$.
Taking negative logarithms gives
\begin{equation}\label{eq:decomp_lyf}
  -\log f(Y\given X)
  =-\log f_Z(Z\given X)
   +\1\{Z=0\}\big(-\log\falt(Y\given X)\big).
\end{equation}
Taking expectations yields
\begin{equation*}
  \Ly(f)=\Lz(f)+\E\left[\1\{Z=0\}\big(-\log\falt(Y\given X)\big)\right].
\end{equation*}
Recalibration replaces $f_Z$ by $g$ and keeps the same alternative distribution, so Equation~\ref{eq:def_recal} gives
\begin{equation}\label{eq:decomp_lyrecal}
\begin{split}
  \Ly\big(\recal{f}\big)
  &=\E[-\log\recal{f}(Y\given X)]\\
  &=\Lz(g,f)+\E\left[\1\{Z=0\}\big(-\log\falt(Y\given X)\big)\right].
\end{split}
\end{equation}
Subtracting the two expected-loss decompositions cancels the common alternative-loss term and proves the result.
\end{proof}

\begin{proof}[Proof of Theorem~\ref{thm:predictor-floor}]
The recalibrated model has budget at most $P+Q$, so the scaling floor and
Proposition~\ref{prop:transfer} give
\begin{equation}
  \Ly^{*}(P+Q)
  \le\Ly\big(\recal{f}\big)
  =\Ly(f)-\big[\Lz(f)-\Lz(g,f)\big].
\end{equation}
Substituting Equation~\ref{eq:excess-loss} and rearranging proves
Equation~\ref{eq:edge-cap}.
\end{proof}

\newpage
\subsection{No Improvement Without a Better Component}
\label{app:better-component}
This section states and proves the result summarized after Proposition~\ref{prop:transfer}.
With an effective predictor, recalibration improves the correctness probabilities while preserving the alternative distribution.
To see what a further improvement would require, consider any model $h(y\given x)$ and split its probabilities according to whether $f$'s answer is correct:
\begin{equation}\label{eq:h-split}
  h_Z(1\given x)=h(\pred(x)\given x),
  \qquad
  h_Z(0\given x)=1-h_Z(1\given x).
\end{equation}
Every model $h$ can be viewed as answering two questions: (i) whether $f$'s prediction is correct and (ii) if it is wrong, which alternative answer is correct.

The quantity $h_Z(1\given x)$ answers the first question.
Once $\pred(x)$ has been ruled out, the distribution ${h(y\given x)}/{h_Z(0\given x)}$ answers the second. 
Define the conditional loss over the alternatives:
\begin{equation}\label{eq:conditional-alt-loss}
  L_{Y\mid Z=0}(h)
  \coloneqq
  \E\left[
    -\log\frac{h(Y\given X)}{h_Z(0\given X)}
    \mathrel{\Big|} Z=0
  \right].
\end{equation}
This is the log loss for the second question, evaluated only when $f$ is wrong.
In binary classification there is only one alternative answer, so $L_{Y\mid Z=0}(h)=0$ for every $h$.
With more than two answers, this term measures how well $h$ distinguishes among the alternatives.

\begin{proposition}[No improvement without a better component]
\label{prop:better-component}
For every model $h$,
\begin{align}
  \Ly(h)
    &=\Lz(h_Z,f)+\Pr(Z=0)L_{Y\mid Z=0}(h),\label{eq:h-components}\\
  \Ly\big(\recal{f}\big)
    &=\Lz(g,f)+\Pr(Z=0)L_{Y\mid Z=0}(f).\label{eq:recal-components}
\end{align}
If the four component losses are finite and $\Ly(h)<\Ly\big(\recal{f}\big)$, then
\begin{equation}\label{eq:better-component}
  \Lz(h_Z,f)<\Lz(g,f)
  \qquad\text{or}\qquad
  L_{Y\mid Z=0}(h)<L_{Y\mid Z=0}(f).
\end{equation}
\end{proposition}

\begin{proof}[Proof of Proposition~\ref{prop:better-component}]
Factor $h$ into its probability that $f$ is correct and its distribution over alternative answers, then take expected negative logarithms to obtain Equation~\ref{eq:h-components}.
Recalibration changes only the correctness probabilities, giving Equation~\ref{eq:recal-components}.
If neither inequality in Equation~\ref{eq:better-component} held, comparing the two decompositions would give $\Ly(h)\ge\Ly\big(\recal{f}\big)$, a contradiction.
\end{proof}

Thus beating recalibration requires better performance predictions, better predictions among the alternative answers, or both.

\newpage
\subsection{Small Predictors Under Diminishing Returns}
\label{app:diminishing-returns}

When the predictor is much smaller than the model, its excess predictability is bounded by the model's suboptimality plus a small capacity correction.
The following corollary makes this precise under diminishing returns, expressed by a convex loss floor.

\begin{corollary}[Small predictors have little advantage beyond model suboptimality]
\label{cor:diminishing-returns}
Under Assumption~\ref{a:floor}, suppose $\Ly^{*}$ is convex.
For $P>0$ and $Q\ge0$, every $f\in\gF_P$ with $\Ly(f)<\infty$ and every $g\in\gG_Q$ satisfy
\begin{equation}\label{eq:diminishing-returns}
  \Lz(f)-\Lz(g,f)
  \le
  \frac{2Q}{P}\big[\Ly^{*}(P/2)-\Ly^{*}(P)\big]
  +\delta_P(f).
\end{equation}
\end{corollary}

The capacity correction is $2Q/P$ times the decrease in the loss floor from $P/2$ to $P$.
When $Q\ll P$, this correction is small, so the excess predictability is bounded essentially by $\delta_P(f)$.
Thus a nearly optimal model leaves little room for a much smaller predictor to outperform its self-confidence.

\begin{proof}[Proof of Corollary~\ref{cor:diminishing-returns}]
By convexity, the loss reduction per parameter over $[P,P+Q]$ is no greater than over $[P/2,P]$, giving
\begin{equation*}
  \Ly^{*}(P)-\Ly^{*}(P+Q)
  \le\frac{2Q}{P}\big[\Ly^{*}(P/2)-\Ly^{*}(P)\big].
\end{equation*}
Substituting this bound into Equation~\ref{eq:edge-cap} proves the result.
\end{proof}

\newpage
\subsection{Suboptimality Does Not Determine Excess Predictability}
\label{sec:tightness}
Proposition~\ref{prop:transfer} shows that a reduction in correctness loss gives an equal reduction in answer loss.
Theorem~\ref{thm:predictor-floor} therefore bounds excess predictability by the answer-loss reduction allowed by capacity gain and suboptimality.

However, lowering answer loss does not necessarily lower correctness loss.
We can decompose answer loss as
\begin{equation}\label{eq:decomp_lyf_expectation}
  \Ly(f)=\Lz(f)+\Pr(Z=0)\Lalt(f),
\end{equation}
where $\Lalt(f)\coloneqq\E[-\log\falt(Y\given X)\given Z=0]$ is the loss over alternative answers when the model's chosen answer is wrong.
A model can lower this alternative loss while keeping its chosen answer and self-confidence unchanged.
At a fixed parameter budget, this reduces suboptimality without changing the correctness-prediction problem.

\begin{proposition}[Lower suboptimality with unchanged maximum excess predictability]
\label{prop:alternative-invariance}
Let $f\in\gF_P$ have finite answer loss and $\Pr(Z=0)>0$.
For $\lambda\in(0,1)$, let $h\in\gF_P$ keep $f$'s point prediction and self-confidence and mix its alternatives with the uniform distribution,
\begin{equation}\label{eq:alternative-mixing}
  h_{\mathrm{Alt}}(y\given x)=(1-\lambda)\falt(y\given x)+\frac{\lambda}{K-1}.
\end{equation}
Then $f$ and $h$ have the same maximum excess predictability at every predictor budget,
\begin{equation}\label{eq:alternative-invariance}
  \sup_{g\in\gG_Q}\big[\Lz(h)-\Lz(g,h)\big]=\sup_{g\in\gG_Q}\big[\Lz(f)-\Lz(g,f)\big],
\end{equation}
while $\delta_P(h)<\delta_P(f)$ if $\Lalt(f)>\log(K-1)$.
\end{proposition}

At fixed parameter budgets, the bound in Theorem~\ref{thm:predictor-floor} decreases with suboptimality, but the maximum excess predictability need not decrease with it.
Models with the same maximum excess predictability can therefore differ in suboptimality by their alternative loss alone.

With two answers, alternative loss is zero, so this source of separation is absent.
As $K$ grows, the alternative distribution has more degrees of freedom, allowing changes in answer loss that correctness prediction does not capture.

\begin{proof}[Proof of Proposition~\ref{prop:alternative-invariance}]
Mixing with a fixed constant is a fixed arithmetic operation, so $h\in\gF_P$; the common point prediction gives the same $Z$, and the common self-confidence gives $\Lz(h)=\Lz(f)$.
For $\lambda<1$, Equation~\ref{eq:alternative-mixing} is invertible, so the outputs of $f$ and $h$ on $x$ determine each other through fixed arithmetic.
Any $g\in\gG_Q$ reading $f$'s output therefore has a counterpart in $\gG_Q$ reading $h$'s output with the same correctness loss, and conversely, which proves Equation~\ref{eq:alternative-invariance}.
Applying Equation~\ref{eq:decomp_lyf_expectation} to $f$ and $h$ and subtracting gives
\begin{equation}\label{eq:alternative-suboptimality}
  \delta_P(f)-\delta_P(h)=\Pr(Z=0)\big[\Lalt(f)-\Lalt(h)\big].
\end{equation}
Since $-\log$ is convex, $\Lalt(h)\le(1-\lambda)\Lalt(f)+\lambda\log(K-1)$, which is below $\Lalt(f)$ when $\Lalt(f)>\log(K-1)$.
\end{proof}

\newpage
\subsection{Routing Induces Performance Prediction}
\label{app:routing-to-prediction}

The following lemma quantifies the prediction gain supplied by a router that captures additional weak-model errors.
It provides the supporting result for the proof of Theorem~\ref{thm:weak-error-capture} in Appendix~\ref{app:proofs-cascading}.

We use the notation of Section~\ref{sec:cascading}, suppressing $X$ in $B_r$, $B_f$, and $u$.
Independent tie-breaking draws are treated as auxiliary inputs, suppressed in the notation and adding no parameters; all expectations include this randomness.
For each model, correctness and the actual error probability are
\begin{equation}\label{eq:routing-error-probability}
  Z_i=\1\{Y=\pred_i(X)\},
  \quad
  \epsilon_i(x)=\Pr(Z_i=0\given X=x),\quad i=1,2.
\end{equation}
At escalation rate $\rho$, the routing decision satisfies
\begin{equation}\label{eq:routing-threshold-decision}
  B_{r,\rho}(x)\in\{0,1\},
  \qquad \E[B_{r,\rho}(X)]=\rho.
\end{equation}
The confidence decision also satisfies $\E[B_f]=\rho$, and we write $u=B_r-B_f$.

At a fixed escalation rate, $B_f$ escalates the lowest-confidence $\rho$ fraction of queries and therefore maximizes the predicted weak errors captured, $\E[B_f f_{1,Z}(0\given X)]$.
Hence
\begin{equation}\label{eq:confidence-routing-optimal}
  \E[u\,f_{1,Z}(0\given X)]\le0.
\end{equation}
As in Equation~\ref{eq:routing-confidence-residual}, additional weak-error capture is therefore bounded by the confidence residual $\E[u(\epsilon_1(X)-f_{1,Z}(0\given X))]$.
To convert this residual into a prediction gain, use the logit correction $g_\lambda$ defined in Equation~\ref{eq:routing-confidence-correction}.
For $\lambda>0$, this raises predicted error where $B_r=1$ and $B_f=0$, and lowers it where $B_r=0$ and $B_f=1$.
At $\lambda=0$, it recovers confidence.

\begin{lemma}[Routing induces performance prediction beyond confidence]
\label{lem:router-to-predictor}
For any routing score $r$ at a fixed rate $\rho$ with $\E[u^2]>0$, the predictors in Equation~\ref{eq:routing-confidence-correction} satisfy
\begin{equation}\label{eq:router-induced-predictor-loss}
  \sup_{\lambda\in\mathbb R}\bigl[\Lz(f_1)-\Lz(g_\lambda,f_1)\bigr]
  \ge\frac{2\bigl(\E[u(\epsilon_1(X)-f_{1,Z}(0\given X))]\bigr)^2}{\E[u^2]}.
\end{equation}
If $\E[u\epsilon_1(X)]>0$, then there exists $\lambda>0$ such that $\Lz(g_\lambda,f_1)<\Lz(f_1)$.
If $\E[u^2]=0$, then $g_\lambda=f_{1,Z}$ almost surely for every $\lambda$.
\end{lemma}

\begin{proof}[Proof of Lemma~\ref{lem:router-to-predictor}]
If $\E[u^2]=0$, then $u=0$ almost surely, so Equation~\ref{eq:routing-confidence-correction} gives $g_\lambda=f_{1,Z}$ for every $\lambda$.
Henceforth assume $\E[u^2]>0$.
We first bound the loss after shifting the confidence logit, then choose the shift that gives the largest guaranteed improvement.

\emph{Step 1: bound the change in loss on each query.}
Fix a query $x$ and a realization of the tie-breaking randomness, which is independent of the answer given $x$.
Write $a(x)=\operatorname{logit}f_{1,Z}(0\given x)$.
For a predicted error logit $s$, let $\sigma(s)=1/(1+e^{-s})$ be the corresponding error probability.
Since the weak model is wrong with probability $\epsilon_1(x)$, the conditional expected binary log loss is
\begin{equation*}
  \begin{aligned}
    \ell_x(s)
    &=-\epsilon_1(x)\log\sigma(s)
      -[1-\epsilon_1(x)]\log[1-\sigma(s)]\\
    &=\log(1+e^s)-\epsilon_1(x)s.
  \end{aligned}
\end{equation*}
Differentiating with respect to $s$ gives
\begin{equation*}
  \ell_x'(s)=\sigma(s)-\epsilon_1(x),
  \qquad
  \ell_x''(s)=\sigma(s)[1-\sigma(s)]\le\frac14.
\end{equation*}
Taylor's theorem therefore bounds the loss at any shift $h\in\mathbb R$ by
\begin{equation*}
  \ell_x(a(x)+h)
  \le\ell_x(a(x))+[f_{1,Z}(0\given x)-\epsilon_1(x)]h+\frac{h^2}{8}.
\end{equation*}
The factor $1/8$ is the Taylor factor $1/2$ times the curvature bound $1/4$.
Substituting $h=\lambda u(x)$ and rearranging yields
\begin{equation*}
  \ell_x(a(x))-\ell_x(a(x)+\lambda u(x))
  \ge\lambda u(x)[\epsilon_1(x)-f_{1,Z}(0\given x)]-\frac{\lambda^2u(x)^2}{8}.
\end{equation*}

\emph{Step 2: average over queries and tie-breaking.}
The two losses on the left are the conditional losses of $f_{1,Z}$ and $g_\lambda$, respectively.
Taking expectations over both queries and tie-breaking thus gives
\begin{equation}\label{eq:routing-correction-quadratic}
  \Lz(f_1)-\Lz(g_\lambda,f_1)
  \ge\lambda\E[u(\epsilon_1(X)-f_{1,Z}(0\given X))]
    -\frac{\lambda^2}{8}\E[u^2].
\end{equation}
These expected losses are finite: $\Lz(f_1)\le\Ly(f_1)<\infty$, and the displayed upper bound, together with $|u|\le1$, bounds the corrected loss for every finite $\lambda$.

\emph{Step 3: choose the correction strength.}
For this optimization, abbreviate the confidence residual and the squared routing difference as
\begin{equation*}
  A=\E[u(\epsilon_1(X)-f_{1,Z}(0\given X))],
  \qquad V=\E[u^2]>0.
\end{equation*}
Completing the square in the right-hand side of Equation~\ref{eq:routing-correction-quadratic} gives
\begin{equation*}
  \lambda A-\frac{\lambda^2V}{8}
  =\frac{2A^2}{V}-\frac{V}{8}\left(\lambda-\frac{4A}{V}\right)^2.
\end{equation*}
Thus $\lambda_*=4A/V$ guarantees excess predictability of at least $2A^2/V$.
The supremum over all $\lambda$ is at least the excess predictability at $\lambda_*$, proving Equation~\ref{eq:router-induced-predictor-loss}.
Finally, if $B_r$ captures more weak-model errors, then Equation~\ref{eq:routing-confidence-residual} gives $A\ge\E[u\epsilon_1(X)]>0$.
Consequently $\lambda_*>0$ and $2A^2/V>0$, so this correction is an effective performance predictor.
\end{proof}

\newpage
\subsection{Proofs for Section~\ref{sec:cascading}}
\label{app:proofs-cascading}
\label{app:weak-error-capture}

Using the notation and supporting lemma from Appendix~\ref{app:routing-to-prediction}, we prove Theorem~\ref{thm:weak-error-capture} by combining the prediction gain with the loss-floor bound from Theorem~\ref{thm:predictor-floor}.

\begin{proof}[Proof of Theorem~\ref{thm:weak-error-capture}]
Recalibration with $g_\lambda$ reuses $f_1$ and adds the budget-$Q$ routing score $r$, the shared escalation rate $\rho$, and the correction strength $\lambda$.
The resulting model has at most $P_1+Q+2$ parameters and belongs to $\gF_{P_1+Q+2}$.
Exact transfer and the loss floor therefore bound the excess predictability for every $\lambda$, as in Theorem~\ref{thm:predictor-floor}:
\begin{equation}\label{eq:routing-predictor-edge-cap}
  \begin{aligned}
    \Lz(f_1)-\Lz(g_\lambda,f_1)
    &\le\delta_{P_1}(f_1)+\Ly^*(P_1)-\Ly^*(P_1+Q+2).
  \end{aligned}
\end{equation}
Write the upper bound on the excess predictability as
\begin{equation*}
  \Delta=\delta_{P_1}(f_1)+\Ly^*(P_1)-\Ly^*(P_1+Q+2).
\end{equation*}
This quantity is nonnegative because suboptimality is nonnegative and the loss floor is non-increasing in the parameter budget.
If $\E[u^2]=0$, the routing decisions agree almost surely and the left side is zero.
Now assume $\E[u^2]>0$.

\emph{Step 1: bound the confidence residual using the excess predictability.}
Lemma~\ref{lem:router-to-predictor} gives a lower bound on the maximum excess predictability obtained by correcting confidence along $u$.
Equation~\ref{eq:routing-predictor-edge-cap} bounds the excess predictability of every such correction by $\Delta$, so it also bounds their supremum.
Together, these bounds give
\begin{equation*}
  \frac{2\bigl(\E[u(\epsilon_1(X)-f_{1,Z}(0\given X))]\bigr)^2}{\E[u^2]}
  \le\sup_{\lambda\in\mathbb R}
    [\Lz(f_1)-\Lz(g_\lambda,f_1)]
  \le\Delta.
\end{equation*}
Multiplying by $\E[u^2]/2$ and taking nonnegative square roots yields
\begin{equation*}
  \left|\E[u(\epsilon_1(X)-f_{1,Z}(0\given X))]\right|
  \le\sqrt{\frac{\E[u^2]}{2}\,\Delta}.
\end{equation*}
By Equation~\ref{eq:routing-confidence-residual}, this also bounds the additional weak errors captured, $\E[u\epsilon_1(X)]$.

\emph{Step 2: bound the difference between routing decisions using the escalation rate.}
Since $B_r,B_f\in\{0,1\}$, we have $u^2=|B_r-B_f|=\1\{B_r\ne B_f\}$.
The absolute difference is at most both $B_r+B_f$ and $(1-B_r)+(1-B_f)$, so
\begin{equation*}
  u^2\le B_r+B_f,
  \qquad
  u^2\le(1-B_r)+(1-B_f).
\end{equation*}
Taking expectations and using $\E[B_r]=\E[B_f]=\rho$ gives both
$\E[u^2]\le2\rho$ and $\E[u^2]\le2(1-\rho)$.
Combining the two steps gives
\begin{equation*}
  \E[u\epsilon_1(X)]
  \le\sqrt{\frac{\E[u^2]}{2}\,\Delta}
  \le\sqrt{\min\{\rho,1-\rho\}\Delta}.
\end{equation*}
Substituting the definition of $\Delta$ proves Equation~\ref{eq:weak-error-capture-bound}.
\end{proof}

\newpage
\subsection{Cascade Advantage over Confidence Routing}
\label{app:cascade-advantage}

To pass from weak-error capture to cascade performance, we must also account for strong-model errors on escalated queries.
The cascade prediction and error rate are
\begin{equation}\label{eq:routing-cascade}
  \hat f_{\mathrm{cas},B_r}(x)
    =(1-B_r(x))\pred_1(x)+B_r(x)\pred_2(x),
\end{equation}
\begin{equation}\label{eq:routing-cascade-error}
  \operatorname{Err}(f_{\mathrm{cas},B_r})
  =\operatorname{Err}(f_1)-\E\big[B_r\big(\epsilon_1(X)-\epsilon_2(X)\big)\big].
\end{equation}
Subtracting this identity for confidence routing gives Equation~\ref{eq:relative-cascade-decomposition}.
If the learned router encounters at least as many strong-model errors as confidence routing, its cascade advantage is at most its additional weak-error capture.

\begin{corollary}[Cascade advantage over confidence routing]\label{cor:cascade-improvement}
Under the conditions of Theorem~\ref{thm:weak-error-capture}, if $\E[u\epsilon_2(X)]\ge0$, then
\begin{equation}\label{eq:cascade-improvement-bound}
  \operatorname{Err}(f_{\mathrm{cas},B_f})-\operatorname{Err}(f_{\mathrm{cas},B_r})
  \le\sqrt{\min\{\rho,1-\rho\}
    \left[\delta_{P_1}(f_1)+\Ly^*(P_1)-\Ly^*(P_1+Q+2)\right]}.
\end{equation}
\end{corollary}

\begin{proof}[Proof of Corollary~\ref{cor:cascade-improvement}]
By Equation~\ref{eq:relative-cascade-decomposition}, the cascade accuracy gain equals $\E[u\epsilon_1(X)]-\E[u\epsilon_2(X)]\le\E[u\epsilon_1(X)]$, which Theorem~\ref{thm:weak-error-capture} bounds.
\end{proof}

In a longer cascade, the analysis applies at each stage to its incoming queries, with $f_2$ representing a fixed downstream policy.

\newpage
\section{Additional Experimental Setup}
\label{app:experiment-details}

\subsection{Models and Data Splits}
\label{app:exp-models}
\label{app:exp-data}

Table~\ref{tab:exp-models} lists the 16 releases evaluated before and after task-specific finetuning.
A \emph{model pair} refers to a release and its selected finetuned checkpoint on one task.

\begin{table}[!htbp]
  \centering
  \small
  \begin{tabular}{ll}
    \hline
    Base release & Instruction-tuned release \\
    \hline
    \model{Qwen3.5-9B-Base} & \model{Qwen3.5-9B} \\
    \model{Gemma-3-12B-PT} & \model{Gemma-3-12B-IT} \\
    \model{Llama-3.1-8B} & \model{Llama-3.1-8B-Instruct} \\
    \model{Apertus-8B-2509} & \model{Apertus-8B-Instruct-2509} \\
    \model{OLMo-3-1025-7B} & \model{OLMo-3-7B-Instruct} \\
    \model{Falcon3-7B-Base} & \model{Falcon3-7B-Instruct} \\
    \model{Mistral-7B-v0.3} & \model{Mistral-7B-Instruct-v0.3} \\
    \model{Mistral-NeMo-Minitron-8B-Base} & \model{Mistral-NeMo-Minitron-8B-Instruct} \\
    \hline
  \end{tabular}
  \caption{Base and instruction-tuned releases from eight model families, giving 16 models per task:
  Qwen3.5 \citep{qwen2026qwen35}, Gemma 3 \citep{gemma2025gemma3}, Llama 3.1 \citep{grattafiori2024llama3},
  Apertus \citep{hernandezcano2025apertus}, OLMo 3 \citep{olmo2025olmo3}, Falcon3 \citep{falcon2024falcon3},
  Mistral 7B \citep{jiang2023mistral}, and Mistral-NeMo-Minitron \citep{sreenivas2024minitron}.}
  \label{tab:exp-models}
\end{table}

Table~\ref{tab:exp-splits} gives the split sizes.
For \dataset{WinoGrande}, \dataset{CommonsenseQA}, and \dataset{GSM8K}, we hold out development questions from the original training set and reserve five remaining questions as demonstrations.
The first two tasks use their original validation sets for testing; \dataset{GSM8K} uses its test set.
We split \dataset{MMLU-Pro}'s 12,032 labeled questions into 6,000 for teacher generation, 3,000 for development, and 3,032 for test, stratified by subject and source.
Its official validation set supplies five demonstrations per subject.
Demonstrations are shared across models and excluded from all other splits; \dataset{MMLU-Pro} metrics pool subjects.
The 85/15 development partition for predictor fitting and selection is shared across model versions and predictor sizes.

\begin{table}[!htbp]
  \centering
  \small
  \setlength{\tabcolsep}{5pt}
  \begin{tabular}{lrrrrr}
    \hline
     & & & \multicolumn{2}{c}{Development split} & \\
    Task & Finetuning & Development & Fit (85\%) & Select (15\%) & Test \\
    \hline
    \dataset{WinoGrande} & 32,313 & 8,080 & 6,868 & 1,212 & 1,267 \\
    \dataset{CommonsenseQA} & 6,394 & 3,342 & 2,840 & 502 & 1,221 \\
    \dataset{GSM8K} & 5,088 & 2,241 & 1,904 & 337 & 1,319 \\
    \dataset{MMLU-Pro} & 5,233 & 3,000 & 2,550 & 450 & 3,032 \\
    \hline
  \end{tabular}
  \caption{Questions used at each stage, excluding demonstrations.
  Predictor fitting and selection partition the development set.
  For \dataset{GSM8K} and \dataset{MMLU-Pro}, finetuning counts include only questions with an accepted teacher solution.}
  \label{tab:exp-splits}
\end{table}

\newpage
\subsection{Answer Scoring and Losses}
\label{app:exp-losses}

\paragraph{Prompts and option scores.}
Base models receive plain-text demonstrations; instruction-tuned models use user/assistant turns in their native chat templates, with thinking disabled.
On \dataset{CommonsenseQA}, the option score is the answer completion's log likelihood, without length normalization.
On \dataset{WinoGrande}, each option fills the blank, and we score the shared continuation under that filled context.

\paragraph{Sampled-answer probabilities.}
Generation uses 128 samples per question at temperature 0.9 and top-$p=1$, with a 2,048-token output limit.
Let $c_x(j)$ count each distinct extracted answer in $S_x$, including a separate category for extraction failures.
We reserve a small mass for unseen answers in category $u$:
\begin{equation}\label{eq:exp-bucket-distribution}
  p(j\given x)=\frac{c_x(j)}{128.5}\quad(j\in S_x),
  \qquad p(u\given x)=\frac{0.5}{128.5}.
\end{equation}
We predict the most frequent successfully extracted answer and use its probability as self-confidence.
An unsampled gold answer receives probability $p(u\given x)$ when computing answer loss.
If all samples fail extraction, the prediction is incorrect with confidence $p(u\given x)$.

\paragraph{Loss conventions.}
All losses are in nats; binary cross-entropy clips probabilities to $[10^{-6},1-10^{-6}]$.
Table~\ref{tab:exp-reference-losses} gives the reference losses for empirical suboptimality (Equation~\ref{eq:exp-endpoints}), held fixed before and after finetuning.

\begin{table}[!htbp]
  \centering
  \small
  \begin{tabular}{lrl}
    \hline
    Task & Loss (nats) & Finetuned model attaining the reference \\
    \hline
    \dataset{WinoGrande} & 0.3229 & \model{Mistral-NeMo-Minitron-8B-Base} \\
    \dataset{CommonsenseQA} & 0.3390 & \model{Qwen3.5-9B-Base} \\
    \dataset{GSM8K} & 0.1378 & \model{Qwen3.5-9B} \\
    \dataset{MMLU-Pro} & 0.7247 & \model{Qwen3.5-9B-Base} \\
    \hline
  \end{tabular}
  \caption{Reference answer losses for empirical suboptimality, computed over the released and finetuned models retained on each task.}
  \label{tab:exp-reference-losses}
\end{table}

\newpage
\subsection{Task-Specific Finetuning and Model Retention}
\label{app:exp-training}
\label{app:exp-model-selection}
\label{app:exp-inclusion}

\paragraph{Teacher solutions.}
For the generation tasks, \model{Qwen3.5-122B-A10B} receives each question without demonstrations and is asked to reason step by step and end with the required answer format.
It generates four solutions at temperature 0.7 and top-$p=0.9$, with a 2,048-token limit.
We keep one whose extracted answer matches the gold answer, omitting questions with no accepted solution.
These targets are generated once and shared across models and seeds.

\paragraph{Optimization and selection.}
LoRA uses rank 8, scaling factor 32, and dropout 0.1 on attention query and value projections; for \model{Qwen3.5-9B}, it also adapts fused query--key--value projections in linear-attention layers.
We train three seeds at learning rate $5\times10^{-5}$, weight decay 0.01, and effective batch size 8.
Only the gold completion or accepted teacher solution contributes to the loss.
Table~\ref{tab:exp-training-settings} gives the schedules and selection criteria.
Greedy accuracy selects generation-task checkpoints; reported results use 128-sample majority answers.

\begin{table}[!htbp]
  \centering
  \small
  \begin{tabular}{lrrl}
    \hline
    Task & Steps & Token limit & Checkpoint selection \\
    \hline
    \dataset{WinoGrande} & 2,000 & 256 & Lowest development answer loss \\
    \dataset{CommonsenseQA} & 1,200 & 512 & Lowest development answer loss \\
    \dataset{GSM8K} & 1,500 & 4,096 & Highest development greedy accuracy \\
    \dataset{MMLU-Pro} & 1,200 & 6,144 & Highest development greedy accuracy \\
    \hline
  \end{tabular}
  \caption{Finetuning settings for the weak models. We save four checkpoints per seed, one at each quarter of training, and select among the 12 candidates from three seeds.}
  \label{tab:exp-training-settings}
\end{table}

\paragraph{Model retention.}

We retain a model pair when finetuning increases development accuracy.
Accuracy uses the highest-scoring option or the 128-sample majority answer, depending on the task.
The rule uses no test labels and retains 16, 16, 14, and 13 pairs in task order.
All reported comparisons, including predictor-size means and empirical loss references, use these retained pairs.

Table~\ref{tab:exp-development-accuracy} shows development accuracy for all models.
Finetuning improves accuracy for every model on the two option-scoring tasks.
Accuracy falls for two instruction-tuned releases on \dataset{GSM8K} and three on \dataset{MMLU-Pro}; all five are among the most accurate releases on their respective tasks.
The illustrative model, \model{Llama-3.1-8B} base, is retained on every task.

\begin{table}[!htbp]
  \centering
  \scriptsize
  \setlength{\tabcolsep}{3pt}
  \renewcommand{\arraystretch}{1.15}
  \begin{tabular}{l@{\hspace{8pt}}rr@{\hspace{9pt}}rr@{\hspace{9pt}}rr@{\hspace{9pt}}rr}
  \hline
   & \multicolumn{2}{c}{\dataset{WinoGrande}} & \multicolumn{2}{c}{\dataset{CommonsenseQA}} & \multicolumn{2}{c}{\dataset{GSM8K}} & \multicolumn{2}{c}{\dataset{MMLU-Pro}} \\
  Model & Rel. & FT & Rel. & FT & Rel. & FT & Rel. & FT \\
  \hline
  \noalign{\vskip 2pt}
  \model{Qwen3.5-9B-Base} & 88.6 & 93.3 & 81.5 & 83.7 & 94.3 & 96.0 & 66.8 & 76.7 \\
  \model{Qwen3.5-9B} & 81.9 & 93.7 & 81.5 & 83.6 & 95.9 & 96.1 & 78.9$^{\dagger}$ & 77.3$^{\dagger}$ \\
  \model{Gemma-3-12B-PT} & 89.5 & 94.2 & 76.5 & 82.4 & 87.1 & 92.7 & 53.5 & 63.0 \\
  \model{Gemma-3-12B-IT} & 82.1 & 93.1 & 76.2 & 80.4 & 93.4 & 95.4 & 64.6$^{\dagger}$ & 62.6$^{\dagger}$ \\
  \model{Llama-3.1-8B} & 88.3 & 93.4 & 73.1 & 78.8 & 70.6 & 82.7 & 41.6 & 50.2 \\
  \model{Llama-3.1-8B-Instruct} & 84.2 & 93.3 & 75.7 & 79.7 & 96.1 & 96.8 & 53.1 & 57.5 \\
  \model{Apertus-8B-2509} & 87.8 & 92.1 & 71.9 & 78.7 & 67.2 & 75.5 & 35.1 & 46.0 \\
  \model{Apertus-8B-Instruct-2509} & 86.3 & 92.4 & 70.2 & 78.2 & 78.0 & 83.0 & 40.4 & 46.2 \\
  \model{OLMo-3-1025-7B} & 86.2 & 91.3 & 71.9 & 80.2 & 90.9 & 94.4 & 41.7 & 58.3 \\
  \model{OLMo-3-7B-Instruct} & 80.5 & 89.8 & 74.8 & 79.5 & 96.8$^{\dagger}$ & 95.9$^{\dagger}$ & 59.6$^{\dagger}$ & 56.2$^{\dagger}$ \\
  \model{Falcon3-7B-Base} & 87.3 & 91.2 & 76.3 & 78.8 & 91.6 & 94.0 & 50.5 & 60.3 \\
  \model{Falcon3-7B-Instruct} & 81.5 & 91.7 & 78.2 & 81.4 & 95.4 & 95.5 & 60.5 & 61.3 \\
  \model{Mistral-7B-v0.3} & 89.5 & 93.9 & 70.3 & 80.1 & 55.3 & 77.6 & 36.8 & 46.0 \\
  \model{Mistral-7B-Instruct-v0.3} & 87.8 & 94.3 & 71.2 & 80.0 & 73.8 & 80.6 & 39.9 & 45.7 \\
  \model{Mistral-NeMo-Minitron-8B-Base} & 89.9 & 94.2 & 84.2 & 86.3 & 81.5 & 88.6 & 42.4 & 53.3 \\
  \model{Mistral-NeMo-Minitron-8B-Instruct} & 82.9 & 92.6 & 80.1 & 85.3 & 97.9$^{\dagger}$ & 96.5$^{\dagger}$ & 49.3 & 54.1 \\
  \hline
  \end{tabular}
  \caption{\textbf{Development accuracy (\%) of released models (Rel.) and their selected finetuned checkpoints (FT).}
  $\dagger$ marks a model whose finetuned checkpoint's development accuracy does not exceed that of its released model.
  Marked pairs are excluded from all reported comparisons, including Figure~\ref{fig:exp-size-ablation}.}
  \label{tab:exp-development-accuracy}
\end{table}

\paragraph{Trajectory experiment.}
For Figure~\ref{fig:exp-trajectory}, we rerun \model{Llama-3.1-8B} base with seed 0 on \dataset{CommonsenseQA} and evaluate every 100 steps through step 600.
Step 0 reuses the original results; later checkpoints each use a new Platt map and the mean loss of three newly trained 4B predictors.

\newpage
\subsection{Performance Predictor Training}
\label{app:exp-predictor}

\paragraph{Inputs.}
The predictor receives the question, the model's final answer, its answer distribution, and a prompt to judge correctness.
On generation tasks, it also receives one sampled reasoning trace ending in the majority answer.
The answer distribution is represented by all option probabilities on the option-scoring tasks, the five most frequent answers on \dataset{GSM8K}, and all sampled option counts on \dataset{MMLU-Pro}.
Inputs exclude demonstrations, chat-template tokens, and gold answers, and are never truncated.

\paragraph{Architecture.}
We encode the input with \model{Qwen3-Embedding-4B} to obtain $\phi(x)$ and apply a linear correction to the self-confidence logit:
\begin{equation}\label{eq:exp-residual}
  g(1\given x)=\sigma\!\left(a+b\,\operatorname{logit}f_Z(1\given x)+\operatorname{head}(\phi(x))\right).
\end{equation}
Here $a,b$ are learned scalars and $\sigma$ is the sigmoid function.

\paragraph{Initialization and training.}
For each evaluated model, we fit Platt calibration~\citep{platt2000probabilistic} by minimizing binary cross-entropy on the predictor-fitting split:
\begin{equation}\label{eq:exp-platt}
  f_Z^{\mathrm{Platt}}(1\given x)
  =\sigma\!\left(\alpha+\beta\,\operatorname{logit}f_Z(1\given x)\right).
\end{equation}
This map is shared across predictor sizes and initializes $(a,b)=(\alpha,\beta)$ in Equation~\ref{eq:exp-residual}; the linear head starts at zero.
The \model{Qwen3-Embedding} backbone uses the final token's hidden state and LoRA on query and value projections with rank 16, scaling factor 32, and dropout 0.05.
We train the adapters, head, and $a,b$ for up to three epochs at learning rate $10^{-4}$ and effective batch size 16, with weight decay $10^{-4}$ on the head only.
Among ten seeds and all epochs, including epoch zero, we select the lowest correctness loss on the predictor-selection split.

\newpage
\subsection{Cascade Evaluation}
\label{app:experiment-cascade}

\paragraph{Strong model.}
We use \model{Qwen3.5-122B-A10B} as the strong model.
On \dataset{WinoGrande} and \dataset{CommonsenseQA}, we finetune it and select the checkpoint with the lowest development answer loss.
These runs use the weak models' training data, prompts, and optimization settings for 2,000 and 1,600 steps, respectively.
We save four checkpoints at quarter intervals and select the lowest development answer loss among them and the released model.
Evaluation uses the weak models' demonstrations and option-scoring rules.

On \dataset{GSM8K} and \dataset{MMLU-Pro}, the released strong model generates 128 samples in thinking mode at temperature 1.0, top-$p=0.95$, top-$k=20$, and presence penalty 1.5, with a 32,768-token limit.
Answer extraction and majority selection follow the weak-model evaluation.

\paragraph{Accuracy and integration.}
For either routing score $r$, we sort the $n$ test questions in decreasing order, breaking ties uniformly at random.
At $\rho=k/n$, we escalate the first $k$ questions.
Expected cascade accuracy is
\begin{equation}\label{eq:exp-cascade-accuracy}
  1-\operatorname{Err}(f_{\mathrm{cas},B_{r,\rho}})=\frac{1}{n}\E\left[
    \sum_{i=1}^n Z_{1,i}
    +\sum_{i=1}^n B_{r,\rho}(x_i)(Z_{2,i}-Z_{1,i})
  \right],
\end{equation}
where $Z_{1,i}$ and $Z_{2,i}$ indicate weak- and strong-model correctness, and the expectation is over tie-breaking with answers held fixed.
We linearly interpolate between attainable rates and integrate with the trapezoidal rule.
Accuracy is linear across each tied group, so its endpoints suffice.
We report routing advantage $100D$ in percentage points averaged over escalation rates.

\newpage
\section{Additional Experimental Results}
\label{app:experiment-results}

All results use the retained pairs; predictors use the selected 4B backbone unless another size is stated.
Values listed in task order refer to \dataset{WinoGrande}, \dataset{CommonsenseQA}, \dataset{GSM8K}, and \dataset{MMLU-Pro}.

\subsection{Correlation Estimates and Uncertainty}
\label{app:exp-uncertainty}

We report Pearson's $r$ and Spearman's $r_s$, with one-sided $p$-values $p_+$ for positive correlation.
We use SciPy's standard tests, assuming independent observations and making no adjustment for multiple comparisons.

\begin{table}[!htbp]
\centering
\small
\setlength{\tabcolsep}{5pt}
\begin{tabular}{llrrrr}
\hline
Comparison & Task & Pearson $r$ & $p_+$ & Spearman $r_s$ & $p_+$ \\
\hline
$\widehat\delta$ vs. $E$ & \dataset{WinoGrande} & 0.98 & $<0.001$ & 0.73 & $<0.001$ \\
 & \dataset{CommonsenseQA} & 0.98 & $<0.001$ & 0.94 & $<0.001$ \\
 & \dataset{GSM8K} & 0.93 & $<0.001$ & 0.94 & $<0.001$ \\
 & \dataset{MMLU-Pro} & 0.65 & $<0.001$ & 0.74 & $<0.001$ \\
\hline
$\Delta\widehat\delta$ vs. $\Delta E$ & \dataset{WinoGrande} & 0.99 & $<0.001$ & 0.90 & $<0.001$ \\
 & \dataset{CommonsenseQA} & 0.99 & $<0.001$ & 0.97 & $<0.001$ \\
 & \dataset{GSM8K} & 0.78 & $<0.001$ & 0.89 & $<0.001$ \\
 & \dataset{MMLU-Pro} & 0.49 & 0.044 & 0.20 & 0.259 \\
\hline
$\widehat\delta$ vs. $D$ & \dataset{WinoGrande} & 0.79 & $<0.001$ & 0.92 & $<0.001$ \\
 & \dataset{CommonsenseQA} & 0.66 & $<0.001$ & 0.87 & $<0.001$ \\
 & \dataset{GSM8K} & 0.91 & $<0.001$ & 0.77 & $<0.001$ \\
 & \dataset{MMLU-Pro} & 0.82 & $<0.001$ & 0.85 & $<0.001$ \\
\hline
$\Delta\widehat\delta$ vs. $\Delta D$ & \dataset{WinoGrande} & 0.63 & 0.004 & 0.76 & $<0.001$ \\
 & \dataset{CommonsenseQA} & 0.42 & 0.054 & 0.56 & 0.011 \\
 & \dataset{GSM8K} & 0.73 & 0.001 & 0.51 & 0.031 \\
 & \dataset{MMLU-Pro} & 0.40 & 0.087 & 0.18 & 0.277 \\
\hline
\end{tabular}
\caption{\textbf{Correlations for the main 4B comparisons.} $p_+$ tests a positive association using the tests described in Appendix~\ref{app:exp-uncertainty}. Endpoint comparisons contain 32, 32, 28, and 26 model versions in task order; paired changes contain 16, 16, 14, and 13 model pairs.}
\label{tab:exp-correlations-main}
\end{table}

\newpage
\subsection{Released and Finetuned Models Separately}
\label{app:exp-groups}

Figures~\ref{fig:exp-predictability-separated} and~\ref{fig:exp-cascade-advantage-separated} repeat Figures~\ref{fig:exp-predictability} and~\ref{fig:exp-cascade-advantage} with released and finetuned models in separate panels, using the same models and the selected 4B predictors.
Among released models alone, excess predictability correlates with empirical suboptimality at 1.00, 0.98, 0.89, and 0.49 in task order, and routing advantage at 0.67, 0.43, 0.91, and 0.81.
The main-text associations therefore do not depend on contrasting the two groups.
Among finetuned models, routing advantage remains positively correlated with suboptimality on every task, at 0.79, 0.81, 0.94, and 0.64, and excess predictability at 0.14, 0.65, 0.99, and 0.32.
The weak value on \dataset{WinoGrande} reflects the narrow range after finetuning: all finetuned models lie within 0.22 nats of the best observed loss, and 14 of 16 have slightly negative excess predictability.
On \dataset{MMLU-Pro}, both groups show the weaker association discussed in Section~\ref{sec:exp-suboptimality}.

\begin{figure}[!htbp]
  \centering
  \includegraphics[width=\linewidth]{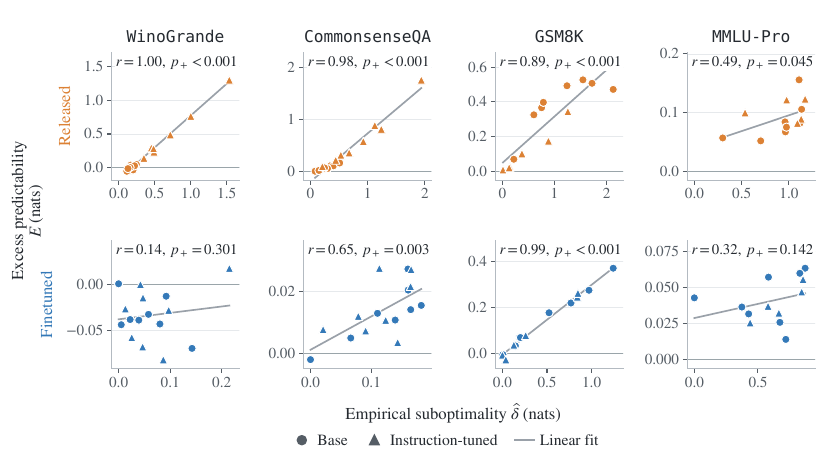}
  \caption{\textbf{Excess predictability of the 4B predictor, with released and finetuned models shown separately.}
  Each column is a task; the top row shows released models and the bottom row their finetuned versions, using the same points as Figure~\ref{fig:exp-predictability}.
  Excess predictability $E$ is the predictor's correctness-log-loss gain over self-confidence; empirical suboptimality $\widehat\delta$ is the gap to the best observed answer log loss on each task.
  Circles denote base models and triangles instruction-tuned models.
  Lines are least-squares fits within each panel; $r$ is Pearson correlation and $p_+$ its one-sided, unadjusted $p$-value for positive correlation.
  Correlations and $p$-values are in Table~\ref{tab:exp-correlations-groups}.
  Axis ranges differ across panels.}
  \label{fig:exp-predictability-separated}
\end{figure}

\begin{figure}[!htbp]
  \centering
  \includegraphics[width=\linewidth]{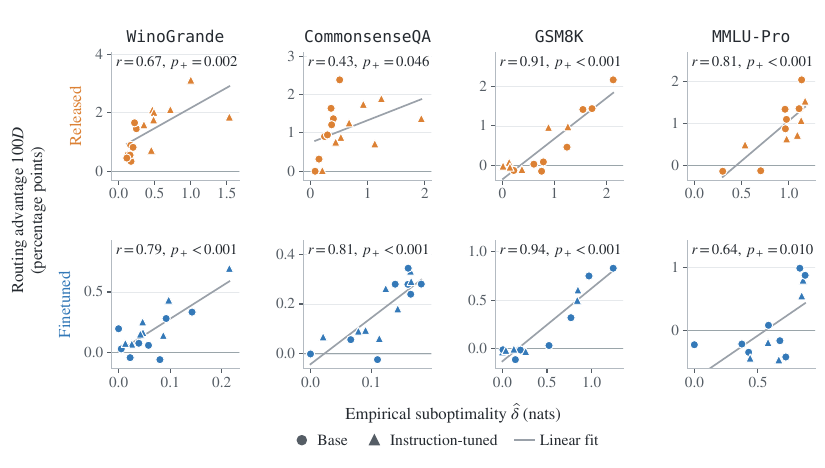}
  \caption{\textbf{Routing advantage of the 4B predictor, with released and finetuned models shown separately.}
  Routing advantage $100D$ (percentage points) versus empirical suboptimality $\widehat\delta$ (nats), using the same points as Figure~\ref{fig:exp-cascade-advantage}.
  $100D$ is the cascade accuracy gain over self-confidence routing, averaged over equal escalation rates; $\widehat\delta$ is the gap to the best observed answer log loss on each task.
  Columns are tasks; orange (top) denotes released models and blue (bottom) their finetuned versions. Circles denote base releases and triangles instruction-tuned releases.
  Lines are least-squares fits; $r$ is Pearson correlation and $p_+$ its one-sided, unadjusted $p$-value for positive correlation. Axis ranges differ across panels.}
  \label{fig:exp-cascade-advantage-separated}
\end{figure}

\begin{table}[!htbp]
\centering
\small
\setlength{\tabcolsep}{5pt}
\begin{tabular}{llrrrr}
\hline
Comparison & Task & Pearson $r$ & $p_+$ & Spearman $r_s$ & $p_+$ \\
\hline
$\widehat\delta$ vs. $E$ (released) & \dataset{WinoGrande} & 1.00 & $<0.001$ & 0.92 & $<0.001$ \\
 & \dataset{CommonsenseQA} & 0.98 & $<0.001$ & 0.97 & $<0.001$ \\
 & \dataset{GSM8K} & 0.89 & $<0.001$ & 0.90 & $<0.001$ \\
 & \dataset{MMLU-Pro} & 0.49 & 0.045 & 0.60 & 0.015 \\
\hline
$\widehat\delta$ vs. $E$ (finetuned) & \dataset{WinoGrande} & 0.14 & 0.301 & -0.03 & 0.547 \\
 & \dataset{CommonsenseQA} & 0.65 & 0.003 & 0.64 & 0.004 \\
 & \dataset{GSM8K} & 0.99 & $<0.001$ & 0.98 & $<0.001$ \\
 & \dataset{MMLU-Pro} & 0.32 & 0.142 & 0.44 & 0.066 \\
\hline
$\widehat\delta$ vs. $D$ (released) & \dataset{WinoGrande} & 0.67 & 0.002 & 0.86 & $<0.001$ \\
 & \dataset{CommonsenseQA} & 0.43 & 0.046 & 0.59 & 0.008 \\
 & \dataset{GSM8K} & 0.91 & $<0.001$ & 0.81 & $<0.001$ \\
 & \dataset{MMLU-Pro} & 0.81 & $<0.001$ & 0.80 & $<0.001$ \\
\hline
$\widehat\delta$ vs. $D$ (finetuned) & \dataset{WinoGrande} & 0.79 & $<0.001$ & 0.56 & 0.011 \\
 & \dataset{CommonsenseQA} & 0.81 & $<0.001$ & 0.82 & $<0.001$ \\
 & \dataset{GSM8K} & 0.94 & $<0.001$ & 0.82 & $<0.001$ \\
 & \dataset{MMLU-Pro} & 0.64 & 0.010 & 0.63 & 0.010 \\
\hline
\end{tabular}
\caption{\textbf{Within-stage correlations for 4B predictors.} $p_+$ tests a positive association using the tests described in Appendix~\ref{app:exp-uncertainty}. Each stage contains 16, 16, 14, and 13 models in task order.}
\label{tab:exp-correlations-groups}
\end{table}

\newpage
\subsection{Predictor Capacity}
\label{app:experiment-capacity}
\label{sec:exp-capacity}

We compare 0.6B, 4B, and 8B predictors to measure how much excess predictability each size captures and whether our main findings persist across predictor sizes.
All three predictors use \model{Qwen3-Embedding} backbones and evaluate the same retained model pairs: 16, 16, 14, and 13 in task order.
We report gains over self-confidence and use two-parameter Platt calibration as an additional baseline.

\paragraph{Excess predictability captured at each size.}
The results are shown in Figure~\ref{fig:exp-size-ablation}.
For released models, larger predictors generally achieve greater gains over self-confidence (first row, orange circles).
After finetuning (first row, blue squares), mean gains decrease on every task at all three sizes, and increasing predictor size also provides less benefit.

The second row compares predictors with Platt-calibrated self-confidence.
The 0.6B predictor adds little beyond calibration, whereas the 4B and 8B predictors generally achieve lower correctness log loss.
For both the 4B and 8B predictors, mean gains over Platt-calibrated self-confidence decrease after finetuning on every task.
For finetuned models on \dataset{WinoGrande}, no predictor size improves on Platt calibration on average.
This is consistent with the finetuned models' low empirical suboptimality on this simple task, leaving little room for additional predictor capacity to help.

\begin{figure}[!htbp]
  \centering
  \includegraphics[width=\linewidth]{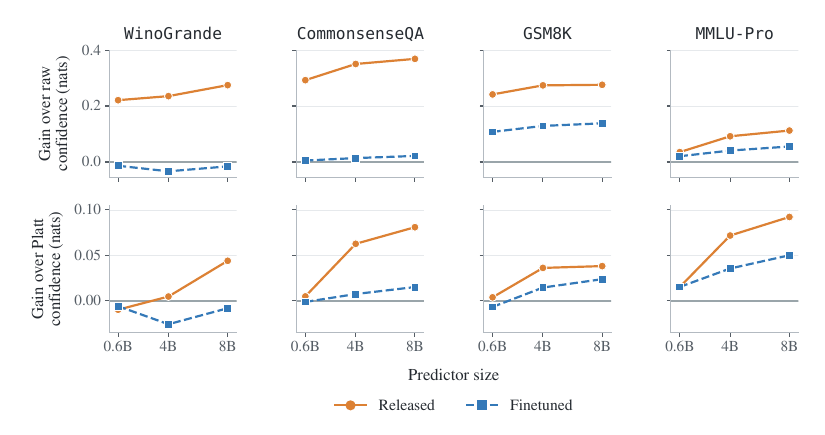}
  \caption{\textbf{Predictor gains over self-confidence and Platt calibration.}
  Mean baseline loss minus predictor loss over the retained models on each task; positive values favor the predictor.
  Columns show tasks; rows compare against raw and Platt-calibrated confidence.
  Orange solid lines with circles: released models; blue dashed lines with squares: finetuned models.
  Predictor size is on the horizontal axis; vertical scales are shared within each row.
  Lines connect the three evaluated sizes.}
  \label{fig:exp-size-ablation}
\end{figure}

\paragraph{Robustness across predictor sizes.}
We check whether predictor size changes the associations with suboptimality across models and the observed effects of finetuning.
Across models, both excess predictability and routing advantage remain positively correlated with suboptimality on every task at 8B, as at 4B (Figures~\ref{fig:exp-predictability-06b}--\ref{fig:exp-cascade-advantage-8b}).
These associations are less consistent at 0.6B: the correlation is negative for excess predictability on \dataset{MMLU-Pro}, and for routing advantage on \dataset{WinoGrande} and \dataset{MMLU-Pro}.
On these tasks, the 0.6B predictor adds little beyond calibration.

The finetuning result holds at all three sizes: finetuning reduces mean excess predictability on every task (Figure~\ref{fig:exp-size-ablation}).
Larger reductions in suboptimality also accompany larger reductions in excess predictability at every size, with Pearson correlations of at least 0.99, 0.98, 0.73, and 0.48 in task order.

\begin{figure}[!htbp]
  \centering
  \includegraphics[width=\linewidth]{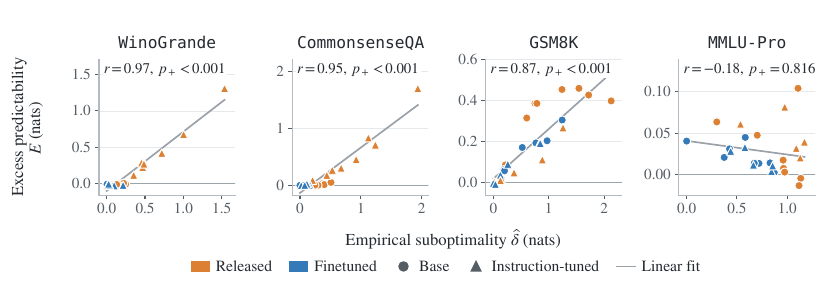}
  \caption{\textbf{Excess predictability with the 0.6B predictor.}
  $E$ is the predictor's correctness-log-loss gain over self-confidence; $\widehat\delta$ is the gap to the best observed answer log loss on each task.
  Orange denotes released models and blue finetuned models; circles denote base releases and triangles instruction-tuned releases.
  Lines are least-squares fits; $r$ is Pearson correlation and $p_+$ its one-sided, unadjusted $p$-value for positive correlation.}
  \label{fig:exp-predictability-06b}
\end{figure}

\begin{figure}[!htbp]
  \centering
  \includegraphics[width=\linewidth]{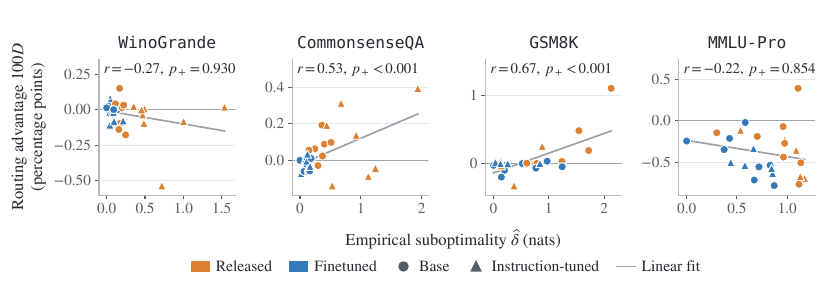}
  \caption{\textbf{Routing advantage with the 0.6B predictor.}
  $100D$ is the cascade accuracy gain over self-confidence routing, averaged over equal escalation rates, in percentage points; $\widehat\delta$ is the gap to the best observed answer log loss on each task.
  Orange denotes released models and blue finetuned models; circles denote base releases and triangles instruction-tuned releases.
  Lines are least-squares fits; $r$ is Pearson correlation and $p_+$ its one-sided, unadjusted $p$-value for positive correlation.}
  \label{fig:exp-cascade-advantage-06b}
\end{figure}

\begin{figure}[!htbp]
  \centering
  \includegraphics[width=\linewidth]{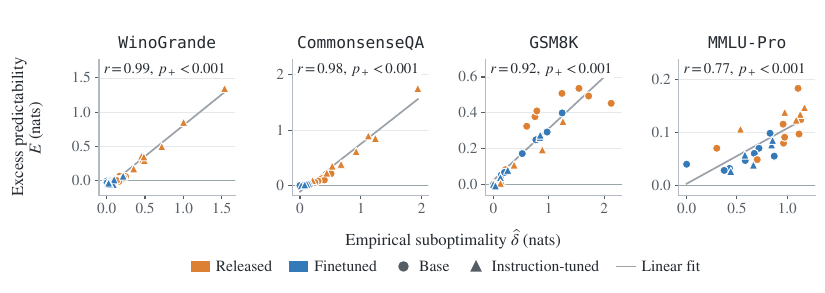}
  \caption{\textbf{Excess predictability with the 8B predictor.}
  $E$ is the predictor's correctness-log-loss gain over self-confidence; $\widehat\delta$ is the gap to the best observed answer log loss on each task.
  Orange denotes released models and blue finetuned models; circles denote base releases and triangles instruction-tuned releases.
  Lines are least-squares fits; $r$ is Pearson correlation and $p_+$ its one-sided, unadjusted $p$-value for positive correlation.}
  \label{fig:exp-predictability-8b}
\end{figure}

\begin{figure}[!htbp]
  \centering
  \includegraphics[width=\linewidth]{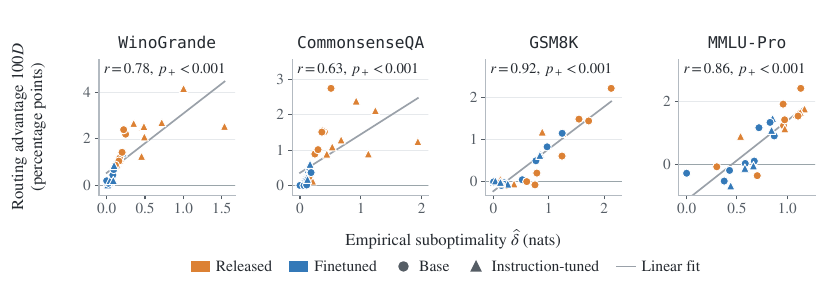}
  \caption{\textbf{Routing advantage with the 8B predictor.}
  $100D$ is the cascade accuracy gain over self-confidence routing, averaged over equal escalation rates, in percentage points; $\widehat\delta$ is the gap to the best observed answer log loss on each task.
  Orange denotes released models and blue finetuned models; circles denote base releases and triangles instruction-tuned releases.
  Lines are least-squares fits; $r$ is Pearson correlation and $p_+$ its one-sided, unadjusted $p$-value for positive correlation.}
  \label{fig:exp-cascade-advantage-8b}
\end{figure}

\clearpage
\newpage
\subsection{Cascade Accuracy Across Tasks}
\label{app:exp-cascade-curves}

Figure~\ref{fig:exp-cascade-curves-four} extends the main-text \dataset{CommonsenseQA} example to all four tasks, using \model{Llama-3.1-8B} base before and after finetuning.
The gap between predictor and self-confidence routing measures the benefit of choosing different questions at the same escalation rate.

\begin{figure}[!htbp]
  \centering
  \includegraphics[width=\linewidth]{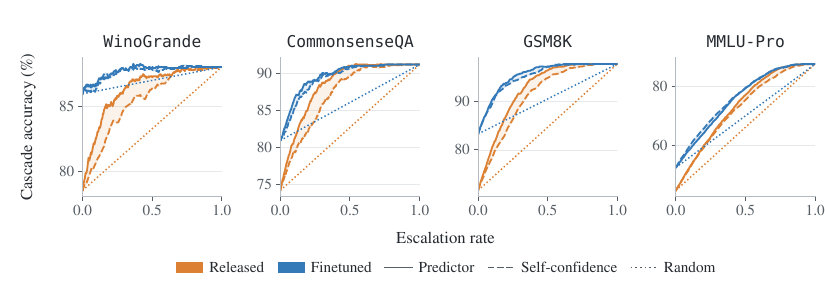}
  \caption{\textbf{Cascade accuracy across all four tasks.}
  We pair \model{Llama-3.1-8B} base with \model{Qwen3.5-122B-A10B}, finetuned on \dataset{WinoGrande} and \dataset{CommonsenseQA} and released on the other two tasks.
  Orange denotes the released weak model and blue its finetuned version.
  Solid curves use the selected 4B predictor $g$, dashed curves use self-confidence $f_{1,Z}$, and dotted lines show random routing.
  Shading shows the accuracy difference between the two routers.}
  \label{fig:exp-cascade-curves-four}
\end{figure}



\newpage
\subsection{How Strong-Model Errors Affect Routing Gains}
\label{app:exp-cascade-results}

\paragraph{Two contributions.}
By Equation~\ref{eq:relative-cascade-decomposition}, the routing advantage at each escalation rate is $\E[u\epsilon_1(X)]-\E[u\epsilon_2(X)]$ with $u=B_{g,\rho}-B_{f,\rho}$.
The first term is the \emph{weak-error term}: weak errors escalated by predictor routing minus those escalated by self-confidence routing.
The second is the \emph{strong-error term}: strong errors encountered by predictor routing minus those encountered by self-confidence routing.
We estimate both on the test questions by replacing $\epsilon_1$ and $\epsilon_2$ with the observed error indicators $1-Z_{1,i}$ and $1-Z_{2,i}$, averaging over tie-breaking as in Equation~\ref{eq:exp-cascade-accuracy}, and integrating over escalation rates as for $D$.
Table~\ref{tab:exp-cascade-split} reports both terms.

\begin{table}[!htbp]
  \centering
  \small
  \begin{tabular}{llrrr}
  \hline
  Task & Models & $100D$ & Weak-error term & Strong-error term \\
  \hline
  \dataset{WinoGrande} & Released & 1.38 & 1.80 & 0.42 \\
   & Finetuned & 0.18 & 0.21 & 0.03 \\
  \hline
  \dataset{CommonsenseQA} & Released & 1.09 & 1.26 & 0.17 \\
   & Finetuned & 0.18 & 0.17 & -0.01 \\
  \hline
  \dataset{GSM8K} & Released & 0.52 & 0.59 & 0.07 \\
   & Finetuned & 0.20 & 0.24 & 0.03 \\
  \hline
  \dataset{MMLU-Pro} & Released & 0.94 & 1.75 & 0.82 \\
   & Finetuned & 0.07 & 0.94 & 0.87 \\
  \hline
  \end{tabular}
  \caption{\textbf{Routing advantage split into two contributions.}
  Means over retained models with 4B predictors, in percentage points averaged over escalation rates.
  Both error terms are differences from self-confidence routing; $100D$ equals the weak-error term minus the strong-error term, up to rounding.}
  \label{tab:exp-cascade-split}
\end{table}

\paragraph{Strong errors usually reduce the gain.}
Predictor routing captures more weak-model errors, but the strong model also answers some of the escalated questions incorrectly.
Averaged over escalation rates, the strong-error term is nonnegative for 91 of 118 retained model versions: in these cases, strong errors offset some of the gain from capturing additional weak errors.
A bootstrap with 500 resamples of test questions gives a 95\% interval entirely above zero in 40 cases and entirely below zero in one.
The effect is largest on \dataset{MMLU-Pro}, where the strong model is 87.6\% accurate.
For finetuned weak models on this task, capturing additional weak errors contributes 0.94 percentage points, but strong errors subtract 0.87 points, leaving a routing advantage of 0.07 points.
On \dataset{GSM8K}, the strong model is 97.9\% accurate, so strong errors have less effect.

\paragraph{Connection to the bound.}
Theorem~\ref{thm:weak-error-capture} bounds the additional weak errors captured in terms of suboptimality and predictor capacity.
At any escalation rate where the strong-error term is nonnegative, Corollary~\ref{cor:cascade-improvement} carries this bound over to the routing advantage.
The 91 cases above satisfy the corresponding comparison after averaging over escalation rates.
Empirically, suboptimality is at least as strongly correlated with the weak-error term as with the routing advantage $D$.
The correlations are 0.80, 0.72, 0.91, and 0.86 in task order, compared with 0.79, 0.66, 0.91, and 0.82 for $D$.
This is consistent with strong errors adding variation to the routing advantage in Figure~\ref{fig:exp-cascade-advantage}.
Avoiding strong errors can also increase the routing advantage, but the largest observed contribution is only 0.12 percentage points.
Without the sign condition, the strong-error contribution to the gain at rate $\rho$ is at most $\min\{\operatorname{Err}(f_2),\rho,1-\rho\}$.
The worst-case allowances averaged over escalation rates are much larger than the measured contributions: 10.6, 8.0, 2.1, and 10.8 points in task order.

\end{document}